\documentclass{article}
\usepackage[final]{neurips_2021}

\usepackage[utf8]{inputenc}
\usepackage{amsmath}
\usepackage{amsthm}
\usepackage{amssymb}
\usepackage{amsfonts}
\usepackage{mathtools}
\usepackage{url}
\usepackage{natbib}
\setcitestyle{square,comma,numbers}
\usepackage{bm}
\usepackage{comment}
\date{}
\renewcommand{\paragraph}[1]{\par\textbf{#1}}
\sloppypar

\usepackage[framemethod=tikz]{mdframed}
\usetikzlibrary{shadows}
\definecolor{captiongray}{HTML}{555555}
\mdfsetup{%
  innertopmargin=1.0ex,
  innerbottommargin=0.0ex,
  innerleftmargin=0.5em,
  innerrightmargin=0.5em,
  linecolor=captiongray,
  linewidth=0.5pt,
  roundcorner=1pt,
  shadow=false
}
\newtheoremstyle{Hogg}
  {0.5\baselineskip}
  {0.5\baselineskip}
  {}
  {}
  {\itshape}
  {:}
  {1ex}
  {}
\theoremstyle{Hogg}
\newtheorem{theorem}{Theorem}
\newtheorem{lemma}[theorem]{Lemma}
\newtheorem{proposition}[theorem]{Proposition}

\newtheorem{definition}{Definition}

\newcommand{\inner}[2]{\langle{#1},{#2}\rangle}

\makeatletter
\renewenvironment{proof}[1][\proofname]{\par
  \vspace{-\topsep}
  \pushQED{\qed}%
  \normalfont
  \topsep0pt \partopsep0pt 
  \trivlist
  \item[\hskip\labelsep
        \itshape
    #1\@addpunct{.}]\ignorespaces
}{%
  \popQED\endtrivlist\@endpefalse
  \addvspace{0pt plus 0pt} 
}
\makeatother

\title{\bfseries%
Scalars are universal: Equivariant machine learning, structured like classical physics}

\author{%
Soledad~Villar\\
Department of Applied Mathematics and Statistics\\
Johns Hopkins University
\And
David~W.~Hogg\\
Flatiron Institute\\
a divison of the Simons Foundation
\And
Kate~Storey-Fisher\\
Center for Cosmology and Particle Physics\\
Department of Physics,
New York University
\And
Weichi~Yao\\
Department of Technology, Operations, and Statistics\\
Stern School of Business,
New York University
\And
Ben~Blum-Smith\\
Center for Data Science\\
New York University
}

\renewcommand{\ldots}{.\,.\,}
\renewcommand{\cdots}{\ldots}

\begin{document}
\maketitle

\begin{abstract}\noindent
    There has been enormous progress in the last few years in designing  neural networks that respect the fundamental symmetries and coordinate freedoms of physical law.
    Some of these frameworks make use of irreducible representations, some make use of high-order tensor objects, and some apply symmetry-enforcing constraints.
    Different physical laws obey different combinations of fundamental symmetries, but a large fraction (possibly all) of classical physics is equivariant to translation, rotation, reflection (parity), boost (relativity), and permutations.
    Here we show that it is simple to parameterize universally approximating polynomial functions that are equivariant under these symmetries, or under the Euclidean, Lorentz, and Poincar\'e groups, at any dimensionality $d$.
    The key observation is that nonlinear O($d$)-equivariant (and related-group-equivariant) functions can be universally expressed in terms of a lightweight collection of scalars---scalar products and scalar contractions of the scalar, vector, and tensor inputs.
    We complement our theory with numerical examples that show that the scalar-based method is simple, efficient, and scalable. 
\end{abstract}

\section{Introduction}
There is a great deal of current interest in building machine-learning methods that respect exact or approximate symmetries, such as translation, rotation, and physical gauge symmetries \cite{cohen2019gauge, kondor2018n}.
Some of this interest is inspired by the great success of convolutional neural networks (CNNs)~\cite{lecun2015deep}, which are naturally translation equivariant.
The implementation of convolutional layers in CNNs has been given significant credit for the success of deep learning, in a domain (natural images) in which the convolutional symmetry is only approximately valid.
In many data-analysis problems in astronomy, physics, and chemistry there are \emph{exact} symmetries that must be obeyed by any generalizable law or rule.
Since the approximate symmetries introduced by convolutional networks help in the natural-image domain, then we have high hopes for the value of encoding exact symmetries for problems where these symmetries are known to hold exactly. %

In detail, the symmetries of physics are legion.
Translation symmetry (including conservation of linear momentum), rotation symmetry (including conservation of angular momentum), and time-translation symmetry (including conservation of energy) are the famous symmetries \cite{noether}.
But there are many more:
there is a form of reflection symmetry (charge-parity-time or CPT);
there is a symplectic symmetry that permits reinterpretations of positions and momenta;
there are Lorentz and Poincar\'e symmetries (the fundamental symmetries of Galilean relativity and special relativity) that include velocity boosts;
there is the generalization of these (general covariance) that applies in curved spacetime;
there are symmetries associated with baryon number, lepton number, flavor, and color;
and there are dimensional and units symmetries (not usually listed as symmetries, but they are) that restrict what kinds of quantities can be multiplied or added.
If it were possible to parameterize a universal or universally approximating function space that is explicitly equivariant under a large set of non-trivial symmetries, that function space would---in some sense---contain within it \emph{every possible law of physics}.
It would also provide a basis for good new machine learning methods.

The most expressive approaches to equivariant machine learning make use of irreducible representations of the relevant symmetry groups. Implementing these approaches requires a way to explicitly decompose tensor products of known representations into their irreducible components. This is called the {\em Clebsch-Gordan problem}. The solution for SO($3$) has been implemented, and there is recent exciting progress for other Lie groups \cite{alex2011numerical, ibort2017new}. This is an area of current, active research.

Here we give an approach that bypasses the need for a solution to this problem.
We find that, for a large class of problems relevant to classical physics, the space of equivariant functions can be constructed from functions only of a complete subset of the scalar products and scalar contractions of the input vectors and tensors.
That is, invariant scalars are powerful objects, and point the way to building universally approximating functions that are constrained by exact, non-trivial symmetries.

\emph{Equivariance} is the form of symmetry core to physical law: For an equivariant function, transforming the input results in an output representation transformed in the same way.
An \emph{invariant} function, on the other hand, produces the same output for transformed and non-transformed inputs.
\begin{definition} Given a function $f:\mathcal X \to \mathcal Y$ and a group $G$ acting on $\mathcal X$ and in $\mathcal Y$ as $\star$ (possibly the action is defined differently in $\mathcal X$ and $\mathcal Y$). We say that $f$ is:
\begin{eqnarray}
    G\text{-invariant: } & f(g\star   x) = f(  x) & \text { for all }   x \in \mathcal X, \; g\in G ~;\\ 
    G\text{-equivariant: } & f(g \star   x) = g\star f(  x) & \text { for all }   x \in \mathcal X, \; g\in G ~.
\end{eqnarray}
\end{definition}

\paragraph{Our contributions:} \label{sec.contributions}
In this work we provide a complete and computationally tractable characterization of all scalar functions $f:(\mathbb R^d)^n \to \mathbb R$, and of all vector functions $h:(\mathbb R^d)^n \to \mathbb R^d$ that satisfy all of the symmetries of classical physics.
The groups corresponding to these symmetries are given in Table~\ref{tab.groups};
they act according to the rules in Table~\ref{tab.actions}.
The characterization we provide is physically principled: It is based on invariant scalars. It is also connected to the symmetries encoded in the Einstein summation rules, a common notation in physics to write expressions compactly but that also allows only equivariant objects to be produced (see Appendix~\ref{sec.einstein}).
\begin{table}[t]
    \begin{mdframed}%
    \centering%
    \begin{tabular}{rl}
    Orthogonal &
    $\text{O}(d) =\{ Q \in \mathbb R^{d\times d}: Q^\top Q = Q\,Q^\top = I_d\}, \label{eq.o}$
    \\[0.5ex]
    Rotation &
    $\text{SO}(d) =\{ Q \in \mathbb R^{d\times d}: Q^\top Q = Q\,Q^\top = I_d, \; \label{eq.so} \operatorname{det}(Q)=1\}$
    \\[0.5ex]
    Translation &
    $\text{T}(d) =\{ w \in \mathbb R^{d} \}$
    \\[0.5ex]
    Euclidean &
    $\text{E}(d) = \text{T}(d) \rtimes \text{O}(d)$
    \\[0.5ex]
    Lorentz &
    $\text{O}(1,d) =\{ Q \in \mathbb R^{(d+1)\times (d+1)}: Q^\top\Lambda\, Q =\Lambda, \,\Lambda=\text{diag}([1,-1,\ldots,-1]) \}$
    \\[0.5ex]
    Poincar\'e &
    $\text{IO}(1,d) = \text{T}(d+1) \rtimes \text{O}(1,d)$
    \\[0.5ex]
    Permutation & $\text{S}_n=\{\sigma:[n]\to [n] \text{ bijective function}\}$
    \\[1ex]
    \end{tabular}
    \caption{\textbf{The groups considered in this work.}}
    \label{tab.groups}
    \end{mdframed}
    \begin{mdframed}%
    \begin{tabular}{rl}
    Orthogonal; Lorentz &
    $Q\star(v_1,\cdots, v_n) = (Q\,v_1, \cdots, Q\,v_n)$
    \\[0.5ex]
    Translation &
    $w\star(v_1,\cdots, v_n) = (v_1 + w, \cdots, v_k + w, v_{k+1},\ldots, v_n)$
    \\
    & (where the first $k$ vectors are position vectors)
    \\[0.5ex]
    Euclidean; Poincar\'e &
    $(w,Q)\star(v_1,\cdots, v_n) = (Q\,v_1 + w, \cdots, Q\,v_k + w, Q\,v_{k+1}, \cdots, Q\,v_n)$ 
    \\[0.5ex]
    Permutation & $\sigma\star(v_1,\ldots, v_n)=(v_{\sigma(1)},\ldots,v_{\sigma(n)})$
    \\[1ex]
    \end{tabular}
    \caption{\textbf{The actions of the groups on vectors.} For the Euclidean group, the position vectors are positions of points; for the Poincar\'e group, the position vectors are positions of \emph{events}.}
    \label{tab.actions}
    \end{mdframed}
\end{table}

Our characterization is based on simple mathematical observations. The first is the First Fundamental Theorem of Invariant Theory for O($d$): \emph{a function of vector inputs returns an invariant scalar if and only if it can be written as a function only of the invariant scalar products of the input vectors} \cite[Section~II.A.9]{weyl}. There is a similar statement for the Lorentz group O($1,d$). The second observation is that \emph{a function of vector inputs returns an equivariant vector if and only if it can be written as a linear combination of invariant scalar functions times the input vectors}.
 In particular, if $h:(\mathbb R^d)^n \to \mathbb R^d$ of inputs $v_1,\ldots,v_n$ is O($d$) or O($1,d$)-equivariant, then it can be expressed as:
\begin{align}
    h(v_1, v_2, \cdots, v_n) &= \sum_{t=1}^n f_t\Big(\inner{v_i}{v_j}_{i,j=1}^n\Big)\,v_t
    ~, \label{od.equivariant}
\end{align}
where $f_t$ can be arbitrary functions, but if $h$ is a polynomial function the $f_t$ can be chosen to be polynomials. In other words, the O($d$) and O($1,d$)-equivariant vector functions are generated as a module over the ring of invariant scalar functions by the projections to each input vector. In this expression, $\inner{\cdot}{\cdot}$ denotes the invariant scalar product, which can be the usual Euclidean inner product, or the Minkowski inner product defined in terms of a metric $\Lambda$ (see Table~\ref{tab.groups}): 
\begin{align} \label{eq.ip}
    \text{Euclidean: } \inner{v_i}{v_j}=v_i^\top v_j~, \quad & \quad \text{Minkowski: }\inner{v_i}{v_j} =v_i^\top \Lambda \,v_j ~.
\end{align}

Our mathematical observations lead to a simple characterization for a very general class of equivariant functions, simpler than any based on irreducible representations or the imposition of symmetries through constraints (these methods are currently state of the art; see Section~\ref{sec.related}).
This implies that very simple neural networks based on scalar products of input vectors---that enormously generalize those in \cite{enequivariant}---can universally approximate invariant and equivariant functions. This justifies the numerical success of the neural network model proposed in \cite{enequivariant}, and mathematically shows that their method can be extended to a universal equivariant architecture with respect to more general group actions. 
We note that the formulation in \eqref{od.equivariant} might superficially resemble an attention mechanism \cite{velivckovic2017graph}, but it actually comes from the characterization of invariant functions of the group, in particular the picture for SO($d$)-invariant functions is a little different (see Proposition \ref{prop.sod}). 

In Section~\ref{sec.related} (and Appendix~\ref{sec:app.irred}) we describe the state of the art for encoding symmetries, the expressive power of graph neural networks, and universal approximation of invariant and equivariant functions.
In Section~\ref{sec.equivariance} we mathematically characterize the invariant and equivariant functions with respect to the groups in Table~\ref{tab.groups}.
In Section~\ref{sec.examples} we present some examples of physically meaningful equivariant functions, and show how to express them in the parameterization developed in Section~\ref{sec.equivariance}.
In Section~\ref{sec.howmany} and Appendix \ref{app.matrix.completion} we discuss which (of all possible) pairwise inner products one ought to provide, and in Section~\ref{sec.limitations} we discuss some limitations of our approach.
We present numerical experiments using our scalar-based approach compared to other methods in Section~\ref{sec.numerical} (see also \cite{yao2021simple}).

We also note that the symmetries considered in this work are all global symmetries, as they act on all points in the same way.
Our characterization thus does not obviously generalize to all gauge symmetries, which are local symmetries that apply changes independently to points at different locations (see \cite{Bronstein2021GeometricDL, weiler2021coordinate}).
That said, we believe our model could be made general enough to encompass gauge symmetries if we replace the global metric by any position-dependent metric $\Lambda_x$. 
In this case, spatially separated vectors would need to be propagated to the same point such that they can be input to locally invariant functions, and this can be done with parallel transport.
The parallel transport operations would also have to obey our invariance characterization, but we believe this is possible to do; we will explore it in future work.

\section{Related work} \label{sec.related}

\paragraph{Group invariant and equivariant neural networks:}
Symmetries have been used successfully for learning functions on images, sets, point clouds, and graphs. Neural networks can be designed to parameterize classes of functions satisfying different forms of symmetries, from the classical (approximately) translation-invariant convolutional neural networks \cite{lecun1989backpropagation}---as well as new approaches that enforce additional symmetries (rotation, scale) on these networks \cite{wang2021incorporating, fuchs2020se, thomas2018tensor}---to more recent architectures that define permutation invariant and equivariant functions on point clouds pioneered by deep sets and pointnets \cite{zaheer2017deep, qi2017pointnet, qi2017pointnet++}, to permutation-equivariant functions on graphs expressed as graph neural networks \cite{gori2005new, scarselli2008graph, kipf2016semi, duvenaud2015convolutional, gilmer2017neural,  chen2019cdsbm}.

For instance, deep sets and pointnets parameterize functions on $(\mathbb R^d)^n $ that are invariant or equivariant with respect to the group of permutations $S_n$ acting  as in Table~\ref{tab.actions}. Invariant theory shows that all invariant and equivariant functions with respect to such actions can be approximated by easily characterized invariant polynomials \cite{weyl}. However, the permutation group can act in significantly more complicated ways. For instance, graph neural networks are equivariant with respect to a different action by permutations (conjugation) that is much harder to characterize (see Appendix~\ref{sec:app.irred}).

Since being introduced by \cite{qi2017pointnet, zaheer2017deep}, deep learning on point clouds has been extremely fruitful, especially in computer vision \cite{guo2020deep}. Recently, new symmetries and invariances have been incorporated into the design of neural networks on point clouds; especially invariances and equivariances with respect to rigid motions such as translations and rotations.
Many architectures have been proposed to satisfy those symmetries such as \cite{thomas2018tensor} based on irreducible representations (irreps), \cite{kondor2018covariant, zhang2019rotation} based on convolutions, \cite{fuchs2020se} employing spherical harmonics and irreps, \cite{zhao2020quaternion} using quaternions, and \cite{finzi2021practical} applying a set of constraints to satisfy the symmetries. Most of the implementations of the approaches mentioned above (except for \cite{finzi2021practical}) are limited to 2D or 3D point clouds.
We provide an overview of the main approaches below and in Appendix~\ref{sec:app.irred}.

Recently \cite{ravanbakhsh2017sharing} developed an approach to enforcing permutation equivariance in neural network layers using parameter sharing, which can then model other symmetry groups. The weight sharing approach is significantly simpler to implement than the ones described above, and it has been proven to be very successful in practice, with several applications \cite{wang2020towards, wang2021incorporating}, including autonomous driving~\cite{huang2021traffic}. Characterizing the space of the functions this approach can express is an interesting open problem. 

\paragraph{Universal approximation via linear invariant layers and irreducible representations:} 
Universally approximating invariant functions can be obtained by taking universal non-invariant functions and averaging them over the group orbits 
\cite{yarotsky2018universal, murphy2019relational}. However, this approach is not practical for large groups like $S_n$ or infinite groups like O($d$).
A classical result in neural networks shows that feed-forward networks with non-polynomial activations can universally approximate continuous functions \cite{leshno1993multilayer}.
This arguably inspired the use of neural networks that are the composition of linear invariant or equivariant layers with compatible non-linear activation functions to create expressive equivariant models \cite{kondor2018n, maron2018invariant, maron2020learning}. Linear $G$-equivariant functions can be written in terms of the irreducible representations of $G$ (we explain this and refer the reader to related literature in Appendix~\ref{sec:app.irred}).
However, the explicit parameterization of the linear maps is only known for a few groups (for instance, SO(3)-equivariant linear maps are parameterized using the Clebsh-Gordan coefficients \cite{fuchs2020se, thomas2018tensor, bogatskiy2020lorentz}). 
Moreover, very recent work \cite{dym2020universality} shows that the classes of functions defined in terms of neural networks over irreducible representations in \cite{fuchs2020se, thomas2018tensor} are universal. In particular, every continous SO(3)-equivariant function can be approximated uniformly in compacts sets by those neural networks. 
Despite universality, there is a limitation to this approach: Although decompositions into irreps are broadly studied in mathematics (also as plethysms), the explicit transformation to the irreps is not known or possible for general groups.  This is in fact an area of current, active research, where there has been recent exciting progress for other Lie groups \cite{alex2011numerical, ibort2017new}, but the implementation is still limited. 

\paragraph{Expressive power of (non-universal) neural networks on graphs:} Graph neural networks express functions on graphs that are equivariant with respect to a certain action by permutations. However, the architectures that are used in practice, typically based on message passing~\cite{gilmer2017neural} or graph convolutions \cite{duvenaud2015convolutional, defferrard2016convolutional}, are not universal in general.
Implementing a universally approximating graph neural network using the formulation from the previous section would be prohibitively expensive.
There is work characterizing the expressive power of message passing neural networks, mainly in terms of the graph isomorphism problem \cite{xu2018powerful, morris2019higher, chen2019equivalence, loukas2019graph, tahmasebi2020counting, chen2020can}, and there is research on the design of graph networks that are expressive enough to perform specific tasks, like solving specific combinatorial optimization problems \cite{bengio2020machine, cappart2021combinatorial, nowak2017note, karalias2020erdos, yao2019experimental,joshi2019efficient, bouritsas2020improving}.

\paragraph{Machine learning for physics with symmetry preservation:}
Machine learning has been applied extensively to problems in physics.
While many of these problems require certain symmetries---in many cases, exact symmetries---most applied projects to date rely on the data to encode the symmetry and hope that the model learns it.
For instance, CNNs are commonly used to classify galaxy images; data augmentation is used to teach the model rotational symmetry \cite{huertas-company2015galaxy, aniyan2017radio, dominguez-sanchez2018galaxy, gonzalez2018augment}.
One well-known example is the Kaggle Galaxy Challenge, a classification competition based on the Galaxy Zoo project \cite{lintott2008gzoo}; the winning model improved performance by concatenating features from transformed images of each galaxy before further training \cite{dieleman2015gzoo}.

There have been recent successes in enforcing physical symmetries in the architecture of the models themselves~\cite{yu-physics}, for instance, in 
weather and climate
modeling~\cite{kashinath2021physics} and in modeling chaotic dynamical systems such as turbulence \cite{wang2020towards, wang2021incorporating}.
In quantum many-body physics, recent work has shown that the symmetries of quantum energy levels on lattices can be enforced with gauge equivariant and invariant neural networks \cite{choo2018quantum, boyda2020hsi, vieijra2020uantum, luo2020quantum, roth2021quantum, luo2021quantum}.
There is significant work on imposing permutation symmetry in jet assignment for high-energy particle collider experiments with self-attention networks \cite{fenton2020jet, lee2020jet}.
In molecular dynamics, rotationally invariant neural networks have been shown to better learn molecular properties \cite{anderson2019cormorant, wang2021molecular, schutt2021equivariant}, and Hamiltonian neural networks have been constructed to better preserve molecular conformations \cite{li2021hamnet}.
More broadly, Hamiltonian networks have been shown to improve physical characteristics, such as better conservation of energy, and to better generalize \cite{greydanus2019hamiltonian, sanchez-gonzalez2019hamiltonian, zhong2020symplectic}, and Lagrangian neural networks can also enforce conservation laws~ \cite{lutter2019lagrangian, cranmer2020lagrangian}.

\paragraph{Invariant theory as a basis for enforcing symmetry in neural networks:} 
We are aware of two lines of prior work that develop approaches related to the one taken here: \cite{ling2016machine, ling2016reynolds} and \cite{gripaios2021lorentz, haddadin2021invariant}.

In \cite{ling2016machine} certain tasks in turbulence modeling and materials science are considered, which have built-in O($3$) and (in the latter case) octahedral symmetry. Each of these problems involves a specific representation of the given symmetry group, for which an explicit generating set for the invariant algebra (there called an ``integrity basis") is known. The authors construct and test models that learn an invariant function, built on these generating sets. In \cite{ling2016reynolds}, the turbulence example is taken up again, this time with the goal of learning an equivariant 2-tensor.

The idea of using the invariant algebra to enforce physical symmetries is also contemplated in \cite{gripaios2021lorentz}. In this work, the authors are focused on developing the underlying invariant theory in the case of simultaneous Lorentz and permutation invariance. In the followup work \cite{haddadin2021invariant}, this idea is developed into three models, using different generating sets, which are tested against each other.

The present work shares with these works the idea of using invariant-theoretic descriptions of invariant functions to hard-code physical symmetries into a neural network that models a physical system. We work in somewhat greater generality. While \cite{ling2016machine, ling2016reynolds} focus on specific (small) representations of O($3$) and finite extensions, and while \cite{gripaios2021lorentz, haddadin2021invariant}  are focused on invariant functions and restricted to linearly reductive groups, we consider both invariant and equivariant functions for almost all groups relevant to physics, including groups (Euclidean and Poincar\'e) that are not reductive, although their invariant theory remains under control.

\paragraph{Inductive bias benefits of incorporating symmetries:} The value of incorporating exact symmetries in machine learning has been recently established empirically in several applications (see for instance \cite{wang2021incorporating}).  Mathematical theory has been developed to explain how much one can improve in terms of sample complexity and generalization~\cite{mei2021learning, bietti2021sample, elesedy2021provably} but many questions remain open.

\section{Equivariant maps}
\label{sec.equivariance}

In this Section we provide a simple characterization of all invariant and equivariant functions with respect to the actions in Table~\ref{tab.actions} by the groups in Table~\ref{tab.groups}.
The proof technique involves the characterization of equivariant maps from knowledge of the invariants, and it is explained in more generality in \cite{blum-smith}.
In what follows, $v_1, v_2, \ldots, v_n$ will be vectors in $\mathbb{R}^d$, $G$ will be a group acting in $\mathbb R^d$ and $(\mathbb R^d)^n$ as in Table~\ref{tab.actions}, $f:(\mathbb R^d)^n \to \mathbb R$ will be an invariant function with respect to the action, and $h:(\mathbb R^d)^n\to \mathbb R^d$ will be an equivariant function with respect to the same action.
$V$ will denote a $d\times n$ matrix whose columns are the $v_i$ vectors.

\paragraph{O($d$) and SO($d$) invariance and equivariance:}
The following classical result in invariant theory (e.g., \cite[Section~II.A.9]{weyl}) shows that O($d$) invariant functions are functions of the scalars $v_i^\top v_j$. 
\begin{lemma}[First Fundamental Theorem for O($d$)]\label{lemma:1}
If $f$ is an O($d$)-invariant scalar function of vector inputs $v_1,\ldots, v_n \in \mathbb R^d$, then $f(v_1,v_2, \ldots, v_n)$ can be written as a function of only the scalar products of the $v_i$.
That is, there is a function $g(\cdot)$ such that
\begin{align}
    f(v_1,v_2, \ldots, v_n) &= g(V^\top V) = g\big((v_i^\top v_j)_{i,j=1}^n\big)
    ~. \label{eq.invariant}
\end{align}
\end{lemma}

\begin{proof}
Given $M=V^\top V \in \mathbb R^{n\times n}$, we can reconstruct $v_1,\ldots, v_n$ modulo the orthogonal group O($d$) by computing the Cholesky decomposition of $M$ (see for instance \cite{trefethen1997numerical} p. 174). Therefore, the function $f$ is uniquely determined by the inner product matrix $V^\top V$.
\end{proof}
 
The classical theorem also includes the fact that if $f$ is polynomial (in the entries of the $v_j$'s), then $g$ can be taken to be polynomial. In Section \ref{sec.howmany} we observe that the function $g$ can be determined by a small subset of the scalars $v_i^\top v_j$.  There is an analogous statement for SO($d$) (again see  \cite[Section~II.A.9]{weyl}):

\begin{lemma}[First Fundamental Theorem for SO($d$)] \label{lem.sodinvars}
If $f$ is an SO($d$)-invariant scalar function of vector inputs $v_1,v_2,\ldots,v_n\in\mathbb R^d$, then $f(v_1,v_2,\ldots, v_n)$ can be written as a function of the scalar products of the $v_i$ and the $d\times d$ subdeterminants of the $d\times n$ matrix $V$.
\end{lemma}

\begin{lemma} \label{lemma.span}
If $h$ is an O($d$)-equivariant vector function of $n$ vector inputs $v_1,v_2, \cdots, v_n$, then $h(v_1,v_2,  \cdots, v_n)$ must lie in the subspace spanned by the input vectors $v_1,v_2, \cdots, v_n$. 
\end{lemma}

\begin{proof}
Let $\{w_1, \ldots, w_r\}\subset \mathbb R^d$ be an orthonormal basis of the orthogonal complement to span($v_1,\ldots, v_n$). Then we can write $h(v_1,\ldots, v_n)=\sum_{t=1}^n \alpha_t v_t + \sum_{t=1}^r \beta_t w_t  $ for some choice of $\alpha_t, \beta_t $. We claim that the equivariance of $h$ implies $\beta_t=0$ for $t=1,\ldots, r$: Consider $\hat Q\in O(d)$ such that $\hat Q(v)=v$ for all span($v_1,\ldots, v_n$), and
$\hat Q(w)=-w$ for all $w$ in the orthogonal complement.

Since $(\hat Qv_1,\ldots, \hat Qv_n) =(v_1,\ldots, v_n)$ we have $h(\hat Q v_1, \ldots, \hat Q v_n)= \sum_{t=1}^d \alpha_t v_t +\sum_{t=1}^r \beta_t w_t $ while $\hat Q(h(v_1, \ldots, v_n))=  \sum_{t=1}^d \alpha_t v_t - \sum_{t=1}^r \beta_t w_t $. Therefore equivariance implies that all $\beta_t=0$.
\end{proof}

Note that Lemma~\ref{lemma.span} doesn't hold for SO($d$): although when the codimension of span($v_1,\dots,v_n$) is 0 or $\geq 2$ a similar argument works, the situation $\operatorname{dim}_{\mathbb R} \operatorname{span}(v_1,\dots,v_n) = d-1$ breaks the proof, and there do exist equivariant vector functions that don't lie in the span when it has codimension 1. For instance, the cross product of two vectors in $\mathbb R^3$ is an SO($3$)-equivariant function that is not in the span of its inputs. Proposition~\ref{prop.sod} shows that generalized cross products are the {\em only} way that SO$(d)$-equivariant vector functions can escape the span of the inputs. 
We further discuss this in Appendix~\ref{sec.einstein}.
Proposition~\ref{prop.invariance} below gives a characterization of all O($d$)-equivariant functions in terms of the scalars. We prove it in Appendix~\ref{app.od}. 
See \cite{blum-smith} for a more general explanation of the proof technique.
\begin{proposition} \label{prop.invariance}
If $h$ is an O($d$)-equivariant vector function of $n$ vector inputs $v_1, v_2, \cdots, v_n$, then there are $n$ O($d$)-invariant scalar functions $f_t(\cdot)$ such that
\begin{align} \label{eq.o.equivariant}
    h(v_1, v_2, \cdots, v_n) &= \textstyle\sum_{t=1}^n f_t(v_1, v_2, \cdots, v_n)\,v_t
    ~. 
\end{align}
Moreover, if $h$ is a polynomial function, the $f_t$ can be chosen to be polynomial.
\end{proposition}
The equivariant scalar functions form a module over the ring of invariant scalar functions. Proposition~\ref{prop.invariance} states that this module is generated by the projections $(v_1,\dots,v_n)\mapsto v_j$. 
Proposition~\ref{prop.sod} extends this result to SO($d$) in terms of the generalized cross product. The definition of the generalized cross product and the proof of Proposition~\ref{prop.sod} are in Appendix~\ref{app.od}.

\begin{proposition} \label{prop.sod}
If $h$ is an SO($d$)-equivariant vector function of $n$ vector inputs $v_1, v_2, \cdots, v_n$ then a similar characterization of Proposition~\ref{prop.invariance} holds (with SO($d$)-invariant scalar coefficients), except when $v_1,v_2,\ldots,v_n$ span a $(d-1)$-dimensional space. In that case, 
there exist SO($d$)-invariant scalar functions $f_t(\cdot)$ and $f_S(\cdot)$ such that
\begin{align}
    h(v_1, v_2, \cdots, v_n) &=  \textstyle\sum_{t=1}^{n} f_t(v_1, v_2, \cdots, v_n)\,v_t + \sum_{S\in \binom{[n]}{d-1}}f_{S}(v_1, v_2, \cdots, v_n)\,v_S
    ~, 
\end{align}
where $[n]:=\{1,\dots, n\}$, $\binom{[n]}{d-1}$ is the set of all $(d-1)$-subsets of $[n]$, and $v_S$ represents the generalized cross product of vectors $v_j$ with $j\in S$ (taken in ascending order). 
Moreover, if $h$ is polynomial, the $f_t$ and $f_S$ can be taken to be polynomial.
\end{proposition}

\paragraph{E($d$) invariance and equivariance:}
When modelling point clouds, we may want to express functions that are translation invariant or equivariant with respect to a subset of the input vectors (for instance, position vectors are translation equivariant).
To this end, we consider the group of translations parametrized by $w\in \mathbb R^d$ acting on the vectors in $(\mathbb R^d)^n$ by translating the position vectors, and leaving everything else unchanged: $w\star(v_1,\ldots, v_k, v_{k+1},\ldots, v_n)= (v_1+w, \ldots, v_k+w, v_{k+1}, \ldots, v_n)$.
In this Section we characterize all functions that are translation and rotation invariant/equivariant. 
In the exposition below we assume for simplicity that all vectors are position vectors, but the results generalize trivially to a mix of vectors.

\begin{lemma} \label{lemma.translation.invariance}
Any translation-invariant function $f:(\mathbb R^d)^n \to \mathbb R^\ell$ with inputs $v_1,\ldots,v_n$ can be written uniquely as
$f(v_1,v_2,\ldots, v_n)=\tilde f(v_2-v_1, \ldots, v_n-v_1)$, where $\tilde f:(\mathbb R^d)^{n-1} \to \mathbb R^\ell$ is an arbitrary function. If $f$ is polynomial, $\tilde f$ is polynomial, and vice versa. If $f$ is equivariant for the action of any given subgroup $G\subset GL(n,\mathbb R)$, then so is $\tilde f$, and vice versa.
\end{lemma}

The proof is given in Appendix~\ref{app.translation}. There is nothing special about subtracting $v_1$; there exist more natural choices to express translation invariant functions. For instance, one classical way in physics to express translation invariance is to take class representatives of the form $(v_1, \ldots, v_n)$ where $\sum v_i =0$ (for example, subtracting the center of mass). Proposition \ref{prop.Ed} characterizes the space of O($d$)-equivariant functions that are also translation invariant or translation equivariant. The proof is in Appendix \ref{app.translation}. 

\begin{proposition} \label{prop.Ed}
An O($d$)-equivariant function $h:(\mathbb R^d)^n\to \mathbb R^d$ that is translation-invariant can be written as \eqref{eq.o.equivariant} where the $f_t$ are O($d$) and translation invariant and $\sum_{t=1}^nf_t(v_1, v_2, \cdots, v_n)=0$. Similarly, if $h$ is translation-equivariant then we can choose $\sum_{t=1}^nf_t(v_1, v_2, \cdots, v_n)=1$.
\end{proposition}

\paragraph{Lorentz symmetry:}
The Lorentz group acts on Minkowski spacetime as Lorentz transformations, which keep the metric tensor invariant.
Lorentz transformations relate space and time between inertial reference frames, which move at a constant relative velocity; spacetime intervals are invariant across frames.
The group is made up of spatial rotations in the three space dimensions and linear velocity ``boosts'' along each dimension.
This set of symmetries---required for special relativity---is not O(4), but rather the non-compact group O(1,3) defined in Table~\ref{tab.groups}.

The characterization of invariant and equivariant functions we obtained for the orthogonal group can be extended to the Lorentz group, obtaining a very similar result, summarized in Proposition~\ref{prop.lorentz} and proven in Appendix~\ref{app.lorentz}.

\begin{proposition} \label{prop.lorentz}
A continuous function $h:(\mathbb R^{d+1})^n \to \mathbb R^{d+1}$ is Lorentz-equivariant (with respect to the action in Table~\ref{tab.actions}) if and only if 
\begin{align}
    h(v_1,\ldots,v_n) = \textstyle\sum_{t=1}^n f_t(v_1, \ldots, v_n)\,v_t
\end{align}
where $f_t:(\mathbb R^{d+1})^n \to \mathbb R$ are Lorentz-invariant scalar functions. Moreover, the functions $f_t$ are uniquely determined by the pairwise Minkowski inner product $\inner{\cdot}{\cdot}$ \eqref{eq.ip}:
\begin{align}
    &f_t(v_1,\ldots, v_n)=g_t( \inner{v_i}{v_j}_{i,j=1}^n ) \label{eq.f}
    ~.
\end{align}
If $h$ is polynomial, the $f_t$ and corresponding $g_t$ can be taken to be polynomial.
\end{proposition}

\paragraph{Poincar\'e symmetry:}
The Poincar\'e group combines translation symmetry with Lorentz symmetry.
Together these complete the symmetries of special relativity, forming the full group of spacetime transformations that preserve the Minkowski metric. 
The generalization from a Lorentz-equivariant formulation to a Poincar\'e-equivariant formulation is similar to the generalization from O($d$) to E($d$):
The \emph{position} vectors take on a special role, in which only \emph{differences} of position can appear as vector inputs to the functions; all other vectors can act unchanged.

Proposition~\ref{prop.poincare} generalizes the results above to the Poincar\'e group action. As before, we assume that all vectors are position vectors for simplicity. The proof is analogous to the proof of Proposition~\ref{prop.Ed}.
\begin{proposition} \label{prop.poincare}
A continuous function $h:(\mathbb R^{d+1})^n \to \mathbb R^{d+1}$ is Poincaré-equivariant (with respect to the action in Table~\ref{tab.actions}) if and only if 
\begin{align}
    h(v_1,\ldots,v_n) = \textstyle\sum_{t=1}^n f_t(v_1, \ldots, v_n)\,v_t
\end{align}
where the $f_t:(\mathbb R^{d+1})^n \to \mathbb R$ are translation and Lorentz-invariant scalar functions, determined as in \eqref{eq.f} by the pairwise Minkowski inner products \eqref{eq.ip}, but also satisfying $\sum_{t=1}^nf_t(v_1, v_2, \cdots, v_n)=1$. Furthermore, if $h$ is polynomial, the $f_t$ and corresponding $g_t$ (as in \eqref{eq.f}) can be taken to be polynomial.
\end{proposition}

\paragraph{Permutation invariance and equivariance:}
Most physics problems are permutation-invariant, in that once you know the masses, sizes, shapes, and so on of the objects in the problem, the physical predictions are invariant to labeling.
In particle physics, this invariance is raised to a fundamental symmetry; fundamental particles (like electrons) are identical and exchangeable.

The characterization of permutation-invariant functions with respect to the action in Table~\ref{tab.actions} is classical \cite[pp.~36--39]{weyl}. Here we prove an extension that describes permutation-invariant vector functions that are also O($d$)-equivariant. The proof is in Appendix~\ref{app.permutations}.

\begin{proposition}\label{lemma.permutations}
Let $h:(\mathbb R^d)^n\to \mathbb R^{d}$ be O($d$)-equivariant (or continuous O($1,d-1$)-equivariant) and also permutation-invariant with respect to the action in Table~\ref{tab.actions}. Then $h$ can be written as
\begin{eqnarray}
   h(v_1,\ldots, v_n)=  \textstyle\sum_{t=1}^n f(v_t, v_1,\ldots,v_{t-1},v_{t+1}, \ldots, v_n)\,v_t, 
\end{eqnarray}
where $f:(\mathbb R^d)^n\to \mathbb R$ is O($d$)-invariant (or  O($1,d-1$)-invariant) and permutation-invariant with respect to the last $n-1$ inputs.
\end{proposition}

Proposition \ref{prop.permutation.equivariance} proven in Appendix \ref{app.permutations} extends the characterization above to permutation equivariant functions (see also \cite{gripaios2021lorentz, klicpera2021gemnet}). 
\begin{proposition} \label{prop.permutation.equivariance}
Let $h:(\mathbb R^d)^n\to (\mathbb R^{d})^n$ be O($d$)-equivariant (or continuous O($1,d-1$)-equivariant) and also permutation-equivariant with respect to the action in Table~\ref{tab.actions}. Then $h$ can be written as $h=(h_1, \ldots, h_n)$ where each $h_i:(\mathbb R^d)^n\to \mathbb R^{d}$ is O($d$)-equivariant (or continuous O($1,d-1$)-equivariant) and
\begin{align}
   h_i(v_1,\ldots, v_n)=  \textstyle\sum_{t=1}^n f_t^{(i)}(v_1,\ldots, v_n)\,v_t,
   \end{align}
   where
all the $f_j^{(i)}:(\mathbb R^d)^n\to \mathbb R$ are O($d$)-invariant (or  O($1,d-1$)-invariant), and for all $i,j=1,\ldots, n$  and all $\sigma\in S_n$ we have
\begin{align}
f^{(i)}_{\sigma^{-1}(j)}(v_{\sigma(1)}, \ldots, v_{\sigma(n)})= f^{(\sigma(i))}_j(v_1,\ldots, \ldots, v_n).
\end{align}
\end{proposition}

\section{Examples}\label{sec.examples}

Here we briefly state two classical physics expressions that obey all the symmetries, and show how to formulate them in terms of invariant scalars.

\paragraph{Total mechanical energy:}
In Newtonian gravity, the total mechanical energy $T$ of $n$ particles with scalar masses $m_i$, vector positions $r_i$, and vector velocities $v_i$ is a scalar\footnote{The energy is \emph{not} a scalar in special relativity or general relativity, but it is a scalar in Newtonian physics.} function:
\begin{align}\label{eq.energy}
    T &= \frac{1}{2}\,\textstyle\sum_{i=1}^n m_i\,|v_i|^2 - \frac{1}{2}\,\sum_{i=1}^n \sum_{\substack{j=1\\j\ne i}}^n \frac{G\,m_i\,m_j}{|r_i - r_j|}
    ~,
\end{align}
where $G$ is Newton's constant (a fundamental constant, and hence scalar).
Since $|a|\equiv(a^\top a)^{1/2}$, this expression \eqref{eq.energy} is manifestly constructed from functions only of scalars $m_i$ and scalar products of vectors.
It is also worthy of note that the positions $r_i$ only appear in differences of position.

\paragraph{Electromagnetic force law:}
The total electromagnetic force $F$ acting on a test particle of charge, 3-vector position, and 3-vector velocity $(q,r,v)$ given a set of $n$ other charges $(q_i,r_i,v_i)$ is a vector function \cite{jackson}:
\begin{align}\label{eq.biot}
    F &= \underbrace{\textstyle\sum_{i=1}^n k\,q\,q_i\,\frac{(r - r_i)}{|r - r_i|^3}}_{\text{electrostatic force}} + \underbrace{\textstyle\sum_{i=1}^n k\,q\,q_i\,\frac{v\times(v_i\times(r - r_i))}{c^2\,|r - r_i|^3}}_{\text{magnetic force}}
    ~,
\end{align}
where $k$ is an electromagnetic constant, $c$ is the speed of light, and $a\times b$ represents the cross product that produces a pseudo-vector perpendicular to vectors $a$ and $b$ according to the right-hand rule.
This doesn't obviously obey our equivariance requirements, because cross products deliver parity-violating pseudo-vectors; these can't be written in O($3$)-equivariant form.
However, a cross of a cross of vectors is a vector, so this expression is in fact O($3$)-equivariant, as are all forces (because forces must be O($3$)-equivariant vectors in order for the theory to be self-consistent).

We can expand the vector triple product using the identity $a \times (b \times c) = (a^\top c)\,b - (a^\top b)\,c$:
\begin{align}
    F &= \sum_{i=1}^n k\,q\,q_i\,\frac{(r - r_i)}{|r - r_i|^3} + \sum_{i=1}^n k\,q\,q_i\,\frac{(v^\top (r - r_i))\,v_i - (v^\top v_i)\,(r - r_i)}{c^2\,|r - r_i|^3} \nonumber \\
    &= \sum_{i=1}^n k\,q\,q_i\,\left( 1 - \frac{v^\top v_i}{c^2} \right)\,\frac{(r - r_i)}{|r - r_i|^3} + \sum_{i=1}^n k\,q\,q_i\,\frac{(v^\top (r - r_i))\,v_i}{c^2\,|r - r_i|^3} 
    ~,\label{eq.elec_force}
\end{align}
where the quantity $v^\top v_i$ is the scalar product of the velocities.
All of the quantities are now straightforwardly functions of invariant scalar products times the input vectors.

\section{How many scalars are needed?}
\label{sec.howmany}
Our analysis in Section~\ref{sec.equivariance} shows that the invariant and equivariant functions of interest (under actions in Table \ref{tab.actions} from groups in Table \ref{tab.groups}) with input vectors $v_1,\ldots, v_n$, can be expressed in terms of the scalars $\inner{v_i}{v_j}_{i\geq j=1}^n$ \eqref{eq.ip}. This greatly simplifies the parameterization of such functions, but it significantly increases the number of features if $n\gg d$. 
In Appendix \ref{app.matrix.completion} we remark that the scalars can be uniquely determined by a small subset of size approximately $(d+1)\, n$. This is related to the rigidity theory of Gram matrices \cite{roth1981rigid} that answers when there exists a unique set of vectors that realize a partial set of distances, and it is closely related to the low rank matrix completion problem \cite{singer2010uniqueness}. Furthermore, there is a vast literature studying high probability robust reconstruction of all scalars from a random subset 
via convex optimization techniques \cite{candes2008exact}. Recently developed optimization techniques on Gram matrices could provide efficient algorithms to learn invariant and equivariant functions \cite{jalali2017variational}.

\section{Limitations and caveats} \label{sec.limitations}

Although the principal results presented here work for many groups, and work naturally at all spatial dimensions $d$ (unlike methods based on irreps, for example), they do not solve all problems for all use cases of equivariant machine learning.
For one, there are myriad groups---and especially discrete groups---that apply to physical and chemical systems where invariant and equivariant functions do not have such a nice characterization. Such is the case for the GNNs discussed in Appendix \ref{sec:app.irred}. 

One example of a situation in which our formulation might not be practical is provided by multipole expansions (for example, those used in the fast multipole method \cite{beatson1997short} and $n$-body networks \cite{kondor2018n}).
In the fast multipole method, a hierarchical spatial graph is constructed, and high-order tensors are used to aggregate information from lower-level nodes into higher-level nodes.
This aggregation is concise and linear when it is performed using high-order tensors; this aggregation is hard (or maybe impossible) when only scalars can be transmitted, without the use of the irreps of the relevant symmetry groups. This is a research direction we are currently exploring.

Another example that suggests that our universal functions might be cumbersome is the current forms in which classical theories---such as electromagnetism and general relativity---are written.
For example, in Section~\ref{sec.examples} we showed that the electromagnetic force law can be written in the form of functions of scalars and scalar products times vectors, but that is \emph{not} how the theory is traditionally written.
It is traditionally written in terms of the magnetic field (a pseudo-vector) or the electromagnetic tensor (an order-2 tensor).
As another example, general relativity is traditionally written in terms of contractions of an order-4 curvature tensor.
That is, although the theories can in principle be written in the forms we suggest, they will in general be much more concise or simple or clear in forms that make use of higher-order or non-equivariant forms.

Finally, although our results apply to many physically relevant groups, they do not encode all of the symmetries of classical physics.
For example, one critical symmetry is the dimensional or units symmetry:
You cannot add or subtract terms that have different units (positions and forces, for example).
This symmetry or consideration has implications for the construction of valid polynomials.
It also implies that only dimensionless (unitless) scalar quantities can be the arguments of large classes of nonlinear functions, including exponentials or sigmoids.
These additional symmetries must be enforced at present with some additional considerations of network architecture or constraints.

We also note that in this work we have characterized \emph{global} symmetries, which act on each point in the same way.
Our characterization does not obviously generalize to gauge symmetries, which are local symmetries that apply changes independently to spatially separated points (see e.g. \cite{Bronstein2021GeometricDL}).
That said, we believe our model could encompass general gauge symmetries if we replace the local metric by any position-dependent metric $\Lambda_x$.
Recent work has shown that equivariance under gauge symmetries is possible in the realm of convolutional neural networks by defining coordinate-independent kernels \cite{weiler2021coordinate}.
In our case, we would have to propagate spatially separated vectors to the same location in order to pass them to a locally invariant function, requiring the operation of parallel transport.
This operation would then also have to obey our invariance characterization; we believe this is possible and we will detail it in future work.

\section{Numerical experiments}
\begin{figure}[t]
   \centering
   \includegraphics[scale=0.4]{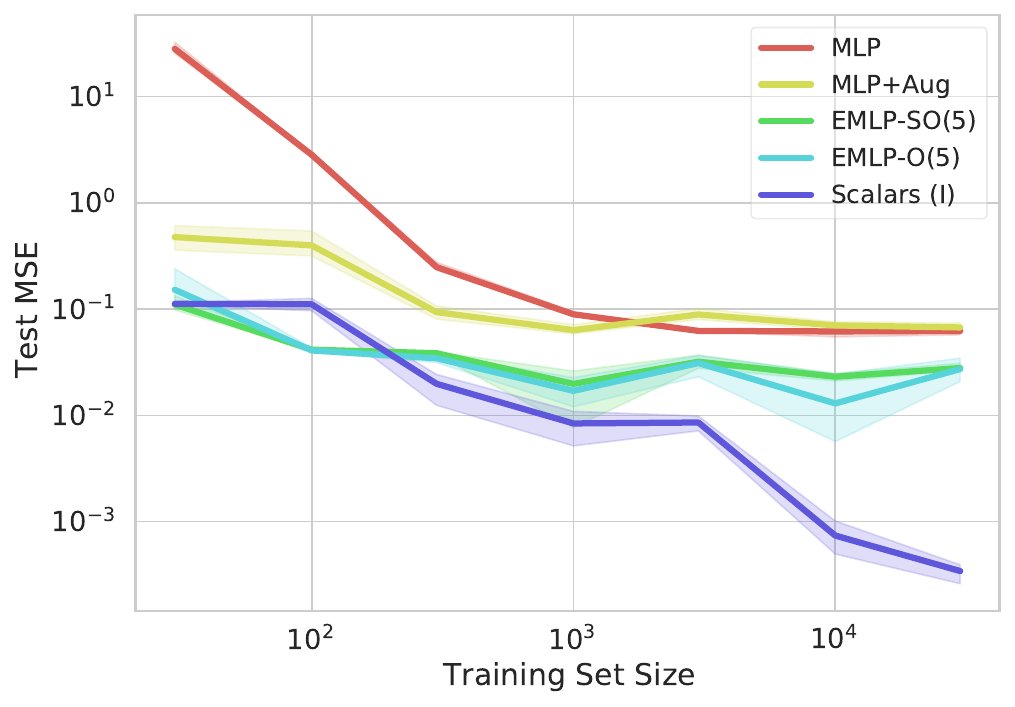}
   \includegraphics[scale=0.4]{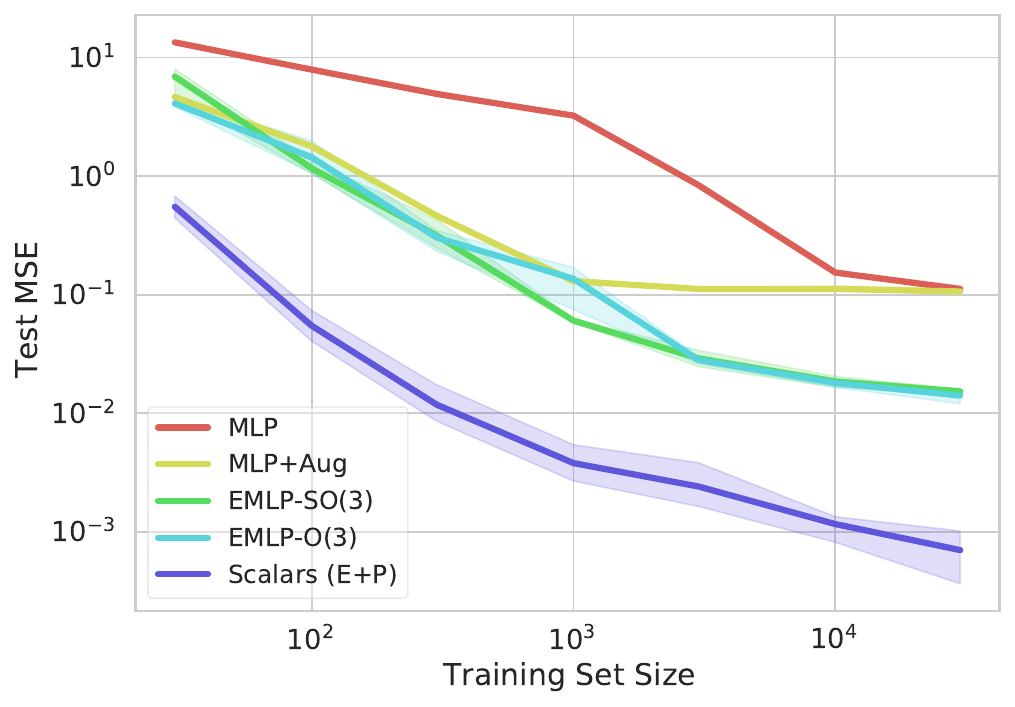}
   \caption{{\small Test error as a function of training set size for (\textbf{Left}) the O(5)-invariant task, and (\textbf{Right}) the O(3)-equivariant task. Scalars (I) denotes the MLP model using the scalars method for O(5)-invariance construction, and Scalars (E+P) denotes the MLP model using the scalars method for O(3)-equivariance and permutation invariance construction.
   MLP denotes a standard multilayer perceptron, and MLP+Aug denotes an MLP that has been trained with data augmentation to the given symmetry group.
   EMLP-G denotes the EMLP models from \cite{finzi2021practical} with different relevant symmetry groups G.
   For both tasks, the scalar method 
   outperforms all other methods. The shaded regions depict 95\% confidence intervals taken over 3 runs.}}
   \label{fig:numerical}
   \vspace{-.4cm}
\end{figure}
\label{sec.numerical}
We demonstrate our approach using scalar-based multi-layer perceptrons (MLP) on two toy learning tasks from \cite{finzi2021practical}: an O(5)-invariant task and an O(3)-equivariant task. Further numerical experiments with these methods applied to dynamical systems appear in \cite{yao2021simple}. The code is available on GitHub\footnote{\url{https://github.com/weichiyao/ScalarEMLP}}, and it reuses much of the functionality provided by EMLP~\cite{finzi2021practical}.

\paragraph{O(5)-invariant task:}
Given observations of the form $(x_1^i, x_2^i, f(x_1^i, x_2^i))_{i=1}^N$, where $x_1^i,x_2^i \in \mathbb R^5$ and  $f$ is an O(5)-invariant scalar function, we aim to learn $f$. In this case $f$ is the example from \cite{finzi2021practical}:
\begin{align}f(x_1,x_2) = \sin(\Vert x_1\Vert)-\frac{\Vert x_2\Vert^3}{2}+\frac{x_1^\top x_2}{\Vert x_1\Vert\Vert x_2\Vert}, \quad x_1,x_2\in\mathbb{R}^5.
\end{align}

We model $f$ using Lemma \ref{lemma:1}, namely 
$f(x_1,x_2) = g(x_1^\top x_1, x_1^\top x_2, x_2^\top x_2)$
where $g:(\mathbb R)^3\to \mathbb R$ could be any function. We learn $g$ by implementing it as an MLP.

\paragraph{O(3)-equivariant task:}
In this task, the data are $n=5$ point masses and positions $(m_i,x_i)_{i=1}^5$, and the goal is to predict the matrix $\mathcal{I} =\sum_{i=1}^5 m_i(x_i^\top x_i I- x_ix_i^\top)$. To this end we aim to learn
\begin{equation}
\begin{aligned}
    h: (\mathbb R\times \mathbb{R}^3)^5   &\to \mathbb R^{3\times 3}  \\
    (m_i,x_i)_{i=1}^5 &\mapsto  \mathcal{I} .
\end{aligned}\label{eq:goal_equiv}
\end{equation} 
The function $h$ is $\mathrm{O}(3)$-equivariant in positions, and $S_5$-invariant, namely, for all $Q\in \mathrm{O}(3)$ and $\sigma \in S_5$ we have
$
    h((m_{\sigma(i)},Qx_{\sigma(i)})_{i=1}^5)=Qh((m_i,x_i)_{i=1}^5).
$
By Proposition \ref{lemma.permutations} we know there exists functions $f_0$, $f_1$, and $f_2$ such that  
\begin{multline*}
    h((m_i,x_i)_{i=1}^5) = \sum_{i=1}^5 f_{0}( x_i^\top x_i, m_i, \{x_k^\top x_l,m_k, m_l\}_{k,l\neq i})\; x_ix_i^\top +
    \\
    +\sum_{i>j=1}^5f_{1}( x_i^\top x_j, m_i,m_j, \{x_k^\top x_l,m_k, m_l\}_{k,l\neq i,j})\; x_ix_j^\top +
    f_2(\{x_i^\top x_j, m_i,m_j\}_{i,j=1}^5) I.
\end{multline*}
Here the set notation means that the function $f_i$ is permutation-invariant with respect to the inputs in the set. The function $h$ is O(3)-equivariant by construction (see Section \ref{sec.equivariance}).
We model $f_0$, $f_1$ ,and $f_2$ with MLPs on deep sets \cite{zaheer2017deep}.

\paragraph{Acknowledgements:}
It is a pleasure to thank Tim Carson (Google), Miles Cranmer (Princeton), Johannes Klicpera (TUM), Risi Kondor (Chicago), Lachlan Lancaster (Princeton), Yuri Tschinkel (NYU), Fedor Bogomolov (NYU), Gregor Kemper (TUM), and Rachel Ward (UT Austin) for valuable comments and discussions. We especially thank Gerald Schwartz (Brandeis) for indicating the line of the polynomial argument to us, and Marc Finzi (NYU) for their help with the EMLP codebase. SV was partially supported by NSF DMS 2044349, the NSF–Simons Research Collaboration on the Mathematical and Scientific Foundations of Deep Learning (MoDL) (NSF DMS 2031985), and the TRIPODS Institute for the Foundations of Graph and Deep Learning at Johns Hopkins University. KSF was supported by the Future Investigators in NASA Earth and Space Science and Technology (FINESST) award number 80NSSC20K1545.

\bibliographystyle{abbrv}
\bibliography{main}

\clearpage
\appendix 

\section{Symmetry-enforcing universal neural networks and irreducible representations} \label{sec:app.irred}

In this Section we describe the general principles to parameterize invariant and equivariant universally approximating functions. As mentioned in Section \ref{sec.related}, the most common approach is is to write the invariant/equivariant functions as the composition of linear invariant/equivariant layers and non-linear compatible pointwise activation functions. 
We start by describing the $G$-invariant graph networks from \cite{maron2019universality}, which
are invariant with respect to a group $G$ 
(typically a subgroup of the symmetric group $S_n$, since the work is in the context of graph neural networks) 
acting on $\mathbb R^n$ as $g\star   x$ ($g\in G$, $  x\in \mathbb R^n$). They extend the action $\star$ to tensors $  t \in \mathbb R^{n^k}$ 
 by acting in each of the $k$ dimensions of the tensor.
 
Consider a graph $ X=( V,  E)$ on $n$ nodes.
Let $ V=(v_1, \ldots, v_n)\in (\mathbb R^d)^n$, where $v_i\in \mathbb R^d$ are the node features, and assume for simplicity that the edge features are real numbers represented by the matrix $E\in \mathbb R^{n\times n}$, then a graph neural network learns  equivariant functions $f:X\to (\mathbb R^{\ell})^n$, where the group of permutations $S_n$ acts in $(\mathbb R^{d})^n$ and $(\mathbb R^{\ell})^n$ as in Table~\ref{tab.actions}, and it acts on the graph $X$ as:
\begin{align} \label{pi.action}
\Pi\star ( V,  E) = (\Pi \star  V, \Pi \star  E) \quad \text{ where } \Pi \star  E = \Pi  E\Pi^\top.
\end{align}

The GNN learns an embedding of the form:
\begin{equation}
    \label{eq.gnn.eq}
\mathcal{N}(  X)= \theta \circ   L_T \circ \ldots.\, \circ   L_2 \circ \theta \circ   L_1 (  E)~,
\end{equation}
where $\theta$ is a point-wise non-linearity and $  L_i: \mathbb R^{n^{\otimes k_i}}\to \mathbb R^{n^{\otimes k_{i+1}}}$ is a linear \emph{equivariant} map.
When $G$ is the group of permutations, the network \eqref{eq.gnn.eq} can universally approximate all continuous equivariant functions as long as the order of the intermediate tensors can be arbitrarily large \cite{maron2019universality, keriven2019universal}. The key insight is that the $k_i$-tensors in the intermediate layers can express all (equivariant) polynomial functions of the input of degree $k_i$.
Universality follows from a generalization of the Stone-Weierstrass theorem that states that every continuous equivariant function defined on a compact set can be uniformly approximated as closely as desired by equivariant polynomial functions (for the specific action by permutations) \cite{keriven2019universal, azizian2020characterizing}.
The disavantage of this approach is that the space of linear equivariant functions $  L_i$, even though it is fully characterized, has dimension that grows super-exponentially with the order of the tensors \cite{maron2018invariant}. 

The approach for general groups is very similar.
Group-equivariant neural networks are written as the composition of linear equivariant functions going to higher order tensors, composed with non-linear pointwise activation functions (for general groups there may be restrictions on the activation functions so that equivariance is preserved).
Recent work proposes to enforce equivariance of these linear maps by imposing constraints \cite{finzi2021practical}.
But in general, the most prevalent approach is to express
intermediate layers as linear maps between (group) representations of $G$. 

Let $G$ be a group acting on $\mathbb R^d$ as $\star$.
In representation-theory language, a representation of $G$ is a map $\rho: G\to \operatorname{GL}(V)$  that satisfies $\rho(g_1g_2)=\rho(g_1)\rho(g_2)$ (where $V$ is a vector space and $\operatorname{GL}(V)$ denotes the automorphisms of $V$, that is, invertible linear maps $V\to V$).
The group action $\star$ of $G$ on $\mathbb R^d$  is equivalent to the group representation $\rho: G \to \operatorname{GL}(\mathbb R^d)$ so that $\rho(g)(v)=g\star v$.
We can extend the action $\star$ to the tensor product $(\mathbb R^d)^{\otimes k}$ so that the group acts independently in every tensor factor (i.e., in every dimension), namely $\rho_k = \otimes_{r=1}^{k} \rho: G \to \operatorname{GL}((\mathbb R^d)^{\otimes k})$. 

The first step is to note that a linear equivariant map $  L_i: (\mathbb R^d)^{\otimes k_i} \to (\mathbb R^d)^{\otimes k_{i+1}}$ corresponds to a map between group representations such that $  L_i \circ \rho_{k_i}(g) = \rho_{k_{i+1}}(g)\circ   L_i$ for all $g\in G$.
Homomorphisms between group representations are easily parametrizable if we decompose the representations in terms of irreps: 
\begin{align}
\rho_{k_i} &= \bigoplus_{\ell=1}^{T_{k_i}} \mathcal T_{\ell} ~.
\end{align}
In particular, Schur's Lemma implies that a map between two irreps over $\mathbb C$ is either zero or a multiple of the identity. 

The equivariant neural-network approach consists in decomposing the group representations in terms of irreps and explicitly parameterizing the maps \cite{kondor2018n, thomas2018tensor, fuchs2020se}.
In general it is not clear how to decompose an arbitrary group representation into irreps. However in the case where $G=\text{SO}(3)$, the decomposition of a tensor representation as a sum of irreps is given by the Clebsh-Gordan decomposition: \begin{equation}
 {\otimes_{s=1}^k \rho_s}= \oplus_{\ell=1}^{T} \mathcal T_{\ell} \label{eq.decomposition}
\end{equation}
The Clebsh-Gordan decomposition not only gives the decomposition of the RHS of \eqref{eq.decomposition} but also it gives the explicit change of coordinates. This decomposition is fundamental for implementing the equivariant 3D point-cloud methods defined in \cite{fuchs2020se, thomas2018tensor, bogatskiy2020lorentz}. 
Moreover, very recent work \cite{dym2020universality} shows that the classes of functions defined in \cite{fuchs2020se, thomas2018tensor} are universal, meaning that every continous SO(3)-equivariant function can be approximated uniformly in compacts sets by those neural networks. 
However, there exists a clear limitation to this approach: Even though decompositions into irreps are broadly studied in mathematics (a.k.a. plethysm), the explicit transformation that allows us to write the decomposition of tensor representations into irreps is a hard problem in general. It is called the {\em Clebsch-Gordan problem}. There is exciting, recent progress on this problem for large classes of groups \cite{alex2011numerical, ibort2017new}. The approach taken in the present work sidesteps this problem altogether.

\section{Equivariant functions under rotations and the orthogonal group} \label{app.od}

\paragraph{Generalized cross-product:}
The generalized cross product of $d-1$ vectors in $\mathbb R^d$ is defined to be the Hodge dual of the exterior product of the $d-1$ vectors. Namely, given $v_1,\ldots, v_{d-1}$ the cross product $v_1\times \ldots \times v_{d-1}$ is the unique vector that satisfies that for all $y\in \mathbb R^d$ 
\begin{equation}
    \inner{v_1\times \ldots \times v_{d-1}}{y} = \operatorname{det}(v_1,\ldots, v_{d-1},y),
\end{equation}
where $\inner{\cdot}{\cdot}$ denotes the usual inner product in $\mathbb R^d$, and $\operatorname{det}(v_1,\ldots,v_d)$ corresponds to the determinant of the matrix with rows $v_1,\ldots,v_d$. 

\paragraph{Proof of 
Proposition~\ref{prop.invariance}:}
The purely set-theoretic statement in Proposition~\ref{prop.invariance} and the polynomial statement have independent proofs. 

For the set-theoretic statement, let $h:(\mathbb R^d)^n \rightarrow \mathbb R^d$ be an arbitrary O($d$)-equivariant vector function. Lemma~\ref{lemma.span} shows that for any fixed $n$-tuple of vectors $(v_1,\ldots,v_n)$, there exists an $n$-tuple of real numbers $(a_1,\ldots,a_n)$ such that
\begin{equation}
h(v_1,\dots,v_n) = \sum_{t=1}^n a_tv_t.
\end{equation}
Pick one representative tuple $(v_1,\dots,v_n)$ from each O($d$)-orbit on $(\mathbb R^d)^n$, and find a corresponding tuple $(a_1,\dots,a_n)$ satisfying this equation. (Note that the tuple $(a_1,\dots,a_n)$ thus depends on the orbit; however, this dependence is suppressed in the notation to prevent it from becoming cumbersome.) Define O($d$)-invariant functions $f_t:(\mathbb R^d)^n \rightarrow \mathbb R$ by defining $f_t(v_1,\dots,v_n) = a_t$ for the $a_t$ corresponding to the chosen representative of $(v_1,\dots,v_n)$'s orbit. Then
\begin{equation}
h(v_1,\dots,v_n) = \sum f_t(v_1,\dots,v_n) v_t
\end{equation}
at the chosen orbit representatives, and the O($d$)-equivariance of both sides then implies this equation is satisfied everywhere.

For the polynomial statement, now assume $h:(\mathbb R^d)^n \rightarrow \mathbb R^d$ is a {\em polynomial} O($d$)-equivariant map. Define a new map $\overline h : (\mathbb R^d)^{n+1}\rightarrow \mathbb R$ by
\begin{equation}
\overline h(v_1,\dots,v_n,y) = \inner{h(v_1,\dots,v_n)} {y},
\end{equation}
where $\langle\cdot,\cdot\rangle$ is the usual inner product on $\mathbb R^d$. Then for any $Q\in \text{O}(d)$, we have
\begin{align*}
\overline h(Q\,v_1,\dots,Q\,v_n,Q\,y)&=\inner{Q\,h(v_1,\dots,v_n)}{ Q\,y} \\
&= \inner{ h(v_1,\dots,v_n)}{ y}\\
&= \overline h(v_1,\dots,v_n,y),
\end{align*}
where the first equality is by the definition of $\overline h$ and the O($d$)-equivariance of $h$ and the second is the fact that the inner product is preserved by O($d$). In other words, $\overline h$ is an O($d$)-invariant scalar function with respect to O($d$)'s natural action on $(\mathbb R^d)^{n+1}$.

It follows from the First Fundamental Theorem for O($d$) that $\overline h$ is a polynomial in the inner products $\inner{v_i}{ v_j}$, $\inner{v_t}{ y}$, and $\inner{y}{y}$. On the other hand, by its definition, it is homogeneous of degree 1 in the coordinates of $y$. It follows that $\inner{y}{y}$ does not appear  this polynomial expression for $\overline h$ in terms of the dot products, and furthermore, each term contains some $\inner{v_t}{y}$ with degree 1, and is otherwise composed of $\inner{v_i}{v_j}$'s. Grouping the terms according to which $\inner{v_t}{ y}$ each contains, we get
\begin{equation}
\overline h(v_1,\dots,v_n,y) = \sum_{t=1}^n g_t( \inner{v_i}{v_j}_{i,j=1}^n) \inner{ v_t}{ y}
\end{equation}
for all $v_1,\dots,v_n,y$. Defining $f_t(v_1,\dots,v_n) = g_t(\inner{v_i}{ v_j}_{i,j=1}^n)$ for each $t$, we find the $f_t$'s are invariant polynomials. Unspooling the definition of $\overline h$ on the left side, and using the linearity of the dot product on the right, this becomes
\begin{equation}
\inner{h(v_1,\dots,v_n)}{ y} = \inner{\left(\sum_{t=1}^n f_t(v_1,\dots,v_n) v_t\right)}{ y}.
\end{equation}
As this equation holds for all $y$, and dot product is a nondegenerate bilinear form, we can conclude that
\begin{equation}
h(v_1,\dots,v_n) = \sum_{t=1}^n f_t(v_1,\dots,v_n) v_t,
\end{equation}
with the $f_t$ invariant polynomials, as promised.\qed

\paragraph{Proof of Proposition~\ref{prop.sod}:}
The proof of this proposition runs parallel to the proof of Proposition~\ref{prop.invariance}. For the set-theoretic part of the proposition, we need a substitute for Lemma~\ref{lemma.span} that applies to SO($d$). The needed statement is that if $h:(\mathbb R^d)^n\rightarrow \mathbb R^d$ is SO($d$)-equivariant, and the span of $v_1,\dots,v_n$ has dimension different from $d-1$, then $h(v_1,\dots,v_n)$ lies in $\operatorname{span}(v_1,\dots,v_n)$. We see this as follows:

Let $W=\operatorname{span}(v_1,\dots,v_n)$, and suppose this has dimension $d-m$. The pointwise stabilizer of $W$ in SO($d$), call it $G_W$, acts in the orthogonal complement $W^\perp$ of $W$. Isomorphically, $G_W$ is SO($m$), and its action on $W^\perp$ is SO($m$)'s canonical action on $\mathbb R^m$. If $m\neq 1$, this action is irreducible; in particular, there are no nonzero vectors in $W^\perp$ that are fixed by the whole action, and it follows that there are no vectors in $\mathbb R^d$ lying outside of $W$ that are fixed by $G_W$. On the other hand, because $G_W$ fixes each of $v_1,\dots,v_n$, equivariance of $h$ implies $h(v_1,\dots,v_n)$ is fixed by $G_W$ as well. It follows that $h(v_1,\dots,v_n)\in W$.

From this lemma it follows that for any tuple $(v_1,\dots,v_n)$ with $\dim \operatorname{span}(v_1,\dots,v_n) \neq d-1$, there exists a solution in $(a_1,\dots,a_n)\in\mathbb R^n$ to the equation
\begin{equation}
h(v_1,\dots,v_n) = \sum a_tv_t.
\end{equation}
On the other hand, when $\dim \operatorname{span}(v_1,\dots,v_n) = d-1$, then there must exist $d-1$ linearly independent vectors $v_{i_1},\dots,v_{i_{d-1}}$. In this situation, the generalized cross product $v_{i_1}\times\dots\times v_{i_{d-1}}$ is nonzero, and linearly independent from $v_{i_1},\dots,v_{i_{d-1}}$ (in fact it lies in the orthogonal complement of $\operatorname{span}(v_1,\dots,v_n)$). Thus it and $v_1,\dots,v_n$ span $\mathbb R^d$. So in all cases, i.e., for any tuple $(v_1,\dots,v_n)$, there exists a solution in the $a_t$ and $a_S$ to
\begin{equation}
h(v_1,\dots,v_n) = \sum_{t=1}^n a_tv_t + \sum_{S\in \binom{[n]}{d-1}}a_Sv_S,
\end{equation}
where the notation is as in the statement of Proposition~\ref{prop.sod}.

The proof of the set-theoretic statement now proceeds exactly as for the proof of Proposition~\ref{prop.invariance}: from each SO($d$)-orbit in $(\mathbb R^d)^n$, choose a representative tuple $(v_1,\dots,v_n)$; pick values of $a_t$ and $a_S$ that satisfy the above for this tuple; use them to define invariant functions $f_t(v_1,\dots,v_n)$ and $f_S(v_1,\dots,v_n)$; then the equation
\begin{equation}
h(v_1,\dots,v_n) = \sum_{t=1}^n f_t(v_1,\dots,v_n)v_t + \sum_{S\in \binom{[n]}{d-1}}f_S(v_1,\dots,v_n)v_S
\end{equation}
holds at the chosen representative tuples, and the SO($d$)-equivariance of both sides shows it holds everywhere.

The proof of the polynomial statement is even more directly parallel to the proof of Proposition~\ref{prop.invariance}. As in that proof, given an SO($d$)-equivariant polynomial vector function $h:(\mathbb R^d)^n\rightarrow \mathbb R^d$, define a polynomial scalar function
\begin{equation}
\overline h(v_1,\dots,v_n,y) = \inner{h(v_1,\dots,v_n)}{y}.
\end{equation}
For the same reason as before, it is SO($d$)-invariant. It follows from the First Fundamental Theorem for SO($d$) that $\overline h$ is a polynomial in the dot products $\langle v_i,v_j\rangle$, $\langle v_t,y\rangle$, $\langle y,y\rangle$, and the $d\times d$ subdeterminants $\det(v_{i_1}\dots v_{i_d})$ and $\det(v_{i_1}\dots v_{i_{d-1}}\; y)$. Because it is also homogeneous of degree 1 in the coordinates of $y$, $\langle y,y\rangle$ cannot occur, and every term must contain either a $\langle v_t,y\rangle$ or a $\det(v_{i_1}\dots v_{i_{d-1}}\; y)$ exactly once, and is otherwise a product of $\langle v_i,v_j\rangle$'s and $\det(v_{i_1}\dots v_{i_d})$'s (which is to say, it is otherwise an SO($d$)-invariant function of the $v_j$'s). Grouping according to which $\langle v_t,y\rangle$ or a $\det(v_{i_1}\dots v_{i_{d-1}}\; y)$ each term contains, we get
\begin{equation}
\overline h(v_1,\dots,v_n) = \sum_{t=1}^n f_t(v_1,\dots,v_n)\langle v_t,y\rangle + \sum_{S\in \binom{[n]}{d-1}}f_S(v_1,\dots,v_n)\det(\tilde v_S,y),
\end{equation}
where the $f_t$'s and $f_S$'s are invariant, and where we have abbreviated $\det(v_{i_1},\dots,v_{i_{d-1}},y)$ as $\det(\tilde v_S,y)$ with $S = \{i_1,i_2,\dots,i_{d-1}\}$. (The tilde is to distinguish it from $v_S = v_{i_1}\times\dots\times v_{i_{d-1}}$.) Now,
\begin{equation}
\det(\tilde v_S,y) = \langle v_S,y\rangle,
\end{equation}
by definition of the generalized cross product $v_S$, so by applying the linearity of the dot product on the right, and the definition of $\overline h$ on the left, we get
\begin{equation}
\langle h(v_1,\dots,v_n), y \rangle= \langle \sum_{t=1}^n f_t(v_1,\dots,v_n)v_t + \sum_{S\in \binom{[n]}{d-1}} f_S(v_1,\dots,v_n)v_S, y\rangle.
\end{equation}
As with O($d$), we get the equality claimed in Proposition~\ref{prop.sod} because $y$ is arbitrary and $\langle\cdot,\cdot\rangle$ is a nondegenerate bilinear form.\qed

\section{Translation-invariant functions}\label{app.translation}

\paragraph{Proof of Lemma~\ref{lemma.translation.invariance}:} Consider the map $\Pi:(\mathbb R^d)^n \rightarrow (\mathbb R^d)^{n-1}$ given by $(v_1,v_2,\ldots,v_n)\mapsto (v_2-v_1,\ldots,v_n-v_1)$. This map has a section $i:(\mathbb R^d)^{n-1}\hookrightarrow (\mathbb R^d)^n$ given by $(v_2,\dots,v_n)\mapsto (0,v_2,\dots,v_n)$, i.e., $\Pi\circ i$ is the identity. The fibers of $\Pi$ are exactly the orbits of the translation action. Thus a translation-invariant function $f$ on $(\mathbb R^d)^n$ descends via $\Pi$ to a well-defined function $\tilde f$ on $(\mathbb R^d)^{n-1}$, so that 
\begin{equation}
f = \tilde f \circ \Pi.
\end{equation}
Then 
\begin{equation}
\tilde f = \tilde f \circ \Pi \circ i = f\circ i.
\end{equation}
Because $\Pi,i$ are polynomial, these equations show that if either $f$ or $\tilde f$ is polynomial then so is the other. Because $i,\Pi$ are both equivariant for the action of $GL(n,\mathbb R)$, any equivariance property of either $f$ or $\tilde f$ with respect to any subgroup $G\subset GL(n,\mathbb R)$ is passed to the other.
\qed

\paragraph{Proof of Proposition~\ref{prop.Ed}:}
Let $h:(\mathbb R^d)^n \rightarrow \mathbb R^d$ be translation-invariant and O($d$)-equivariant. Using Lemma~\ref{lemma.translation.invariance} we can consider an O($d$)-equivariant function $\tilde h:(\mathbb R^d)^{n-1}\to \mathbb R^d$ so that $h(v_1,v_2, \ldots, v_n)=\tilde h(v_2-v_1, \ldots, v_n-v_1)$. Therefore there exists  O($d$) invariant functions $\tilde f_t$
\begin{align}
    \tilde h(v_2-v_1, \ldots, v_n-v_1) &= \textstyle \sum_{t=2}^n \tilde f_t(v_2-v_1, \ldots, v_n-v_1)(v_t-v_1). 
    \end{align}
This implies $h$ can be written as 
   \begin{align} 
   h(v_1,v_2, \ldots, v_n)
&=\textstyle\sum_{t=2}^n f_t(v_1,\ldots, v_n)\,v_t -(\textstyle\sum_{t=2}^n f_t(v_1,\ldots, v_n))\,v_1
\end{align}
where the functions $f_t$ are O($d$)- and translation-invariant, which has the claimed form. 

For the case where $h$ is O($d$)- and translation-equivariant we first observe that it suffices to take one representative per orbit, define the function on that representative, and extend it everywhere else by translations: 
\begin{align}
    h(v_1,\ldots, v_n)= \tilde h(v_2-v_1,\ldots, v_n-v_1)+v_1,
\end{align}
where $\tilde h$ is O($d$)-equivariant.
Therefore any function $h$ that is O($d$)- and translation-equivariant can be written as
\begin{align}
    h(v_1,\ldots, v_n)&= \textstyle\sum_{t=2}^n \tilde f_t(v_2-v_1,\ldots, v_n-v_1)(v_t-v_1) + v_1 \\
    &=  \textstyle\sum_{t=2}^n f_t(v_1, v_2\ldots, v_n)\,v_t  + \left(1- \textstyle\sum_{t=2}^n f_t(v_1, v_2\ldots, v_n)\right)\,v_1. 
\end{align}
\qed

\section{Invariant and equivariant functions under the Lorentz group}
\label{app.lorentz}
\paragraph{Proof of Proposition \ref{prop.lorentz}:}
The proof follows the pattern of the proof of Proposition~\ref{prop.invariance}. The claim for polynomial $h$ follows in exactly the same manner, because there is a First Fundamental Theorem for the Lorentz group precisely analogous to Lemma~\ref{lemma:1}---in fact, both are consequences of the First Fundamental Theorem for O($d$,$\mathbb C$) \cite[Proposition~5.2.2]{goodman2009symmetry}, because over $\mathbb C$, the Lorentz group and the orthogonal group are related by a change of basis (namely, multiplying the space coordinates by $i$). 

The claim for arbitrary continuous $h$ has two added subtleties for the Lorentz group. First, unlike for O($d$), the Minkowski inner products of the vectors $v_j$ do not distinguish every pair of distinct O($1$,$d-1$) orbits from each other: for example the orbit of $(v,\dots,v)$, with $\inner{v}{v}=0$, is indistinguishable from $0$ by the inner products, even if $v\neq 0$. So there exist O($1,d-1$)-invariant set functions on $(\mathbb R^d)^n$ that are not functions of the Minkowski inner products. However, they do distinguish every pair of {\em closed} orbits, and every orbit has a closed orbit in its closure (this is a general theorem about the invariant ring of a reductive group, see \cite[Theorems~4.6 and 4.7 and their corollaries]{popov1994invariant}; the Lorentz group is reductive because it is a real form of the reductive group O($d, \mathbb C$)). Thus every {\em continuous} invariant function is indeed a function of the Minkowski inner products.

The second subtlety is in proving that a continuous, equivariant vector function $h$ always lies in the span of $v_1,\dots,v_n$. Our approach in Section~\ref{sec.equivariance} to show that the invariant functions under the orthogonal group are restricted to the span of the input vectors is based on the following idea: Given $v_1,\ldots v_n$, if they span $\mathbb R^d$, then the result trivially holds. Otherwise let $\{w_1,\ldots w_m\}$ be a basis of  $\text{span}(v_1,\ldots,v_n)$ and  extend it to a basis of $\mathbb R^{d}$. 
Then we do a full orthogonalization (via Gram Schmidt) to get orthogonal vectors $u_1,u_2,\cdots,u_d$ and then construct $Q$:
\begin{align}
    u_1 &\leftarrow w_1
    \\
    \mbox{then for each $j$ ($2\leq j\leq d$) in order:} ~~ u_j &\leftarrow w_j - \sum_{k=1}^{j-1} \frac{w_j^\top u_k}{u_k^\top u_k}\,u_k
\end{align}\vspace{-1\baselineskip}\begin{align}
    Q &\leftarrow \left[\sum_{j=1}^{m} \frac{1}{u_j^\top u_j}\,u_j\,u_j^\top\right] - \left[\sum_{j=m+1}^{d} \frac{1}{u_j^\top u_j}\,u_j\,u_j^\top
    \right]
    ~.
\end{align}
Then $Q$ is an element of O($d$) that fixes everything in $\text{span}(v_1,\ldots, v_n)$ and does not fix anything outside it. This can be used to show that equivariant functions are restricted to the span of the inputs (see Lemma \ref{lemma.span}).

This idea can be generalized to the Lorentz group O$(1,d)$. Let $(t,x_1,\ldots, x_d) \in \mathbb R^{d+1}$, the metric is
\begin{align}
    \Lambda &= \hat{e}_t\,\hat{e}_t^\top - \sum_{j=1}^d \hat{e}_{x_j}\,\hat{e}_{x_j}^\top
    ~,
\end{align}
where $\hat{e}_t$ is a timelike unit vector and the $\hat{e}_{x_j}$ are orthonormal space-like vectors (all orthogonal to $\hat{e}_t$).
Given $\{v_1,\ldots v_n\}$ we consider $\{w_1,\ldots w_m\}$ a basis of  $\text{span}(v_1,\ldots,v_n)$ and we extend it to $\{w_1,\ldots w_{d+1}\}$ a basis of $\mathbb R^{d+1}$ and orthogonalize all the vectors according to the Lorentz generalization of orthogonalization given above to make orthogonal vectors $u_1,u_2,\cdots,u_{d+1}$ and then construct $Q$:
\begin{align}
    u_1 &\leftarrow w_1
    \\
    \mbox{then for each $j$ ($2\leq j\leq d+1$) in order:} ~~ u_j &\leftarrow w_j - \sum_{k=1}^{j-1}\frac{\inner{w_j}{u_k}}{\inner{u_k}{u_k}}\,u_k  \label{eq.gs}
\end{align}\vspace{-1\baselineskip}\begin{align}
    Q &\leftarrow \left[\sum_{j=1}^{m} \frac{1}{\inner{u_j}{u_j}}\,u_j\,u_j^\top \Lambda\right] - \left[\sum_{j=m+1}^{d+1} \frac{1}{\inner{u_d}{u_d}}\,u_d\,u_d^\top \Lambda \right]
    ~. \label{eq.q}
\end{align}
Note that \eqref{eq.gs} and \eqref{eq.q} require $\inner{u_j}{u_j}\neq 0$ for all $j$, however, the Minkowski inner product \eqref{eq.ip} is not positive definite, in particular there can be \emph{lightlike} vectors $u_j\neq 0$ in which $\inner{u_j}{u_j}= 0$. 
In our argument we first will require that if $m=1$ then $v_1$ is not lightlike.
If the procedure \eqref{eq.gs} generates a lightlike $u_j$ for any $j$ in the range $1\leq j\leq d+1$ we proceed in the following way:
\begin{itemize}
    \item If $m>1$ and $j\leq m$, replace the vectors $w_1, \ldots, w_n$ by a linear combination of them with random coefficients and start again. The fact that the ligthlike vectors are a codimension 1 conic surface ensures that every linear subspace of dimension greater than or equal to 2 has an orthogonal basis where none of the basis elements are lightlike.
    \item If $j>m$ and $d\geq m+1$ then the complement of the span of $w_1,\ldots,w_m$ has dimension at least 2 and the same procedure can be applied: one can remove lightlike vectors without changing the subspace by replacing $w_{m+1},\ldots,w_{d+1}$ with a linear combination of them.
    \item If $m=d$ the procedure will not produce a lightlike vector as its last vector. This is because an orthogonal basis cannot have a lighlike vector. Each lightlike vector is orthogonal to $d$ linearly independent vectors but one of those is itself. So if $u_{d+1}$ is a lightlike vector $w_{d+1}$ is in the span of $w_1,\ldots w_d$ which contradicts our assumption that $\{w_1,\ldots, w_{d+1}\}$ is linearly independent.
\end{itemize}
This shows that given $\{v_1,\ldots, v_n\}$ set of vectors spanning a linear space of dimension $m>1$, there exists a $Q\in$O($1,d$) such that $Q(v)=v$ for all $v\in \text{span}(v_1,\ldots,v_n)$ and $Q(w)\neq w$ for all $w\not \in \text{span}(v_1,\ldots,v_n)$, which can be used to show that $h(v_1,\ldots,v_n)\in \text{span}(v_1,\ldots,v_n)$ as in Lemma~\ref{lemma.span}.

If $m=1$ and $v_1$ is lightlike, the continuity of $h$ and the fact that the set of non-lightlike vectors are dense imply that $h(v_1)\in\text{span}(v_1)$. Without the continuity of $h$ this argument wouldn't work. This is because the set of lightlike vectors is invariant under the action of the Lorentz group.\qed

\section{Permutation-invariant and equivariant functions that are also orthogonal or Lorentz-equivariant} \label{app.permutations}

\paragraph{Proof of Proposition \ref{lemma.permutations}:}
Let $\Pi_{j\to i}=\{\sigma \in S_n: \, \sigma(j)=i\}$. We start by writing  $h$ as in \eqref{od.equivariant}. Using the invariance we average over the orbit of $S_n$ we obtain:
\begin{align}
    h(v_1,\ldots, v_n) &=
    \sum_{j=1}^n f_j(v_1, \ldots, v_n) v_j
    \\
    &= \frac{1}{n!} \sum_{j=1}^n \sum_{\sigma \in S_n} f_j(v_{\sigma(1)}, \ldots, v_{\sigma(n)}) v_{\sigma(j)}
    \\
    &= \frac{1}{n!} \sum_{j=1}^n \left( \sum_{i=1}^n \sum_{\sigma \in \Pi_{j\to i}} f_j(v_{\sigma(1)}, \ldots, v_{\sigma(n)})v_i\right)
    \\
    &=\sum_{i=1}^n \frac{1}{n!}\left( \sum_{j=1}^n \sum_{\sigma \in \Pi_{j\to i}} f_j(v_{\sigma(1)}, \ldots, v_{\sigma(n)})\right)v_i
\end{align}
Let $\tilde f_i(v_i,v_1,\ldots,v_{i-1},v_{i+1}, \ldots, v_n) = \frac{1}{n!} \left( \sum_{j=1}^n \sum_{\sigma \in \Pi_{j\to i}} f_j(v_{\sigma(1)}, \ldots, v_{\sigma(n)})\right)$. For fixed $j$, the set of permutations $\Pi_{j\to i}$ is stable (as a set) under post-composition with any permutation $\tau$ that fixes $i$. As a consequence, each summand $\sum_{\sigma \in \Pi_{j\to i}} f_j(v_{\sigma(1)}, \ldots, v_{\sigma(n)})$ is invariant under the action of any such permutation $\tau$. As a consequence, $\tilde f_i$ is invariant with respect to the last $n-1$ inputs, and $h$ can be expressed as
\begin{equation}
    h(v_1,\ldots, v_n) = \sum_{i=1}^n \tilde f_i(v_{i}, v_{[-i]})\,v_i,
\end{equation}
using the notation $v_{[-t]}:=(v_1,\ldots,v_{t-1},v_{t+1}, \ldots, v_n)$.
We now show that all $\tilde f_i$'s can be chosen to the same functions. In order to do so average over permutations again:

\begin{align}
     h(v_1, \ldots,  v_n) &= \frac{1}{n!} \sum_{i=1}^n \sum_{\sigma \in S_n} \tilde f_i(v_{\sigma(i)}, v_{[-\sigma(i)]}) v_{\sigma(i)},
     \\
     & = \frac{1}{n!} \sum_{i=1}^n \left(\sum_{j=1}^n \sum_{\sigma \in \Pi_{i\to j}} \tilde f_i(v_j, v_{[-j]}) \right) v_{j},
     \\
     & = \frac{1}{n!}\sum_{j=1}^n \left(\sum_{i=1}^n (n-1)! \, \tilde f_i(v_j, v_{[-j]}) \right) v_{j}, \\
     & = \sum_{j=1}^n \left( \frac{1}{n}\sum_{i=1}^n \tilde f_i(v_j, v_{[-j]}) \right) v_{j}, \label{eq.perm.same}
\end{align}
Therefore we define $\hat f$ such as
\begin{equation}
    \hat f(w, w_1,\ldots, w_{n-1}) = \frac{1}{n} \sum_{i=1}^n \tilde f_{i}(w, w_1,\ldots, w_{n-1}),
\end{equation}
a $O(d)$-invariant function that is permutation invariant with respect to the last $n-1$ inputs. The computation in \eqref{eq.perm.same} shows that: 
\begin{equation}
h(v_1, \ldots,  v_n) =
   \sum_{j=1}^n  \hat f(v_j, v_{[-j]}) v_{j},
\end{equation}
which proves the theorem. \qed

\paragraph{Proof of Proposition \ref{prop.permutation.equivariance}:}
Let $\sigma \in S_n$ and $(h_1,\ldots, h_n)=h:(\mathbb R^d)^n\to (\mathbb R^d)^n$ a  permutation-equivariant function. Then 
\begin{align}
h(\sigma\star (v_1, \ldots, v_n)) &= 
\sigma \star h(v_1,\ldots, v_n) \\
h(v_{\sigma(1)}, \ldots, v_{\sigma(n)})&=
(h_{\sigma(1)}(v_1,\ldots, v_n), \ldots, h_{\sigma(n)}(v_1,\ldots, v_n) ) \\
(h_1 (v_{\sigma(1)}, \ldots, v_{\sigma(n)}), \ldots, h_n(v_{\sigma(1)}, \ldots, v_{\sigma(n)}) ) &=
(h_{\sigma(1)}(v_1,\ldots, v_n), \ldots, h_{\sigma(n)}(v_1,\ldots, v_n) ) \label{eq.equiv}
\end{align}

Since $h_i: (\mathbb R^d)^n\to \mathbb R^d$ are 
O($d$)-equivariant (or continuous O($1,d-1$)-equivariant) we have that for all $i$ there exists $f_t^{i}:(\mathbb R^d)^n\to \mathbb R$ O($d$)-invariant (or O($1,d-1$)-invariant) such that:
\begin{align}
    h_i(v_1,\ldots, v_n) = \sum_{t=1}^n f_t^{(i)}(v_1,\ldots, v_n) v_t. \label{eq.lincomb}
\end{align}
Combining \eqref{eq.equiv} and \eqref{eq.lincomb} we get
\begin{align}
    h_i(v_\sigma(1),\ldots, v_\sigma(n)) &= \sum_{t=1}^n f_t^{(i)}(v_{\sigma(1)}, \ldots, v_{\sigma(n)}) v_{\sigma(t)} \label{eq.sigma}\\
    h_{\sigma(i)}(v_1,\ldots, v_n)
    &= \sum_{t=1}^n f_t^{(\sigma(i))}(v_1,\ldots, v_n) v_{t} \label{eq.sigma2} .
\end{align}
Taking $t=\sigma^{-1}(j)$ in \eqref{eq.sigma} we get $t=j$ in \eqref{eq.sigma2} and matching coefficients we get
\begin{align}
f^{(i)}_{\sigma^{-1}(j)}(v_{\sigma(1)}, \ldots, v_{\sigma(n)})= f^{(\sigma(i))}_j(v_1,\ldots, \ldots, v_n).
\end{align}
\qed

\section{Einstein summation notation}\label{sec.einstein}
We can interpret the results from Section~\ref{sec.equivariance} as coming from the symmetries encoded in the Einstein summation rules.

An important early realization in differential geometry and in general relativity was that an enormous class of generally covariant forms can be written in a form that is commonly known (incorrectly perhaps) as ``Einstein summation notation'' \cite{einstein}; it is a subset of the prior Ricci calculus \cite{ricci}.
This notation is a method for finding and checking equivariant quantities useful in physical laws:
We imagine that we have O($d$)-equivariant vectors $u$, $v$, $w$, and each has $d$ components such that $[u]_i$ is the $i$th component of $u$.
If we write products with repeated indices like $[u]_i\,[v]_i$, the sum rule is that the repeated index $i$ is summed, so this corresponds to a O($d$)-invariant scalar product (dot product or inner product) of the vectors $u$ and $v$. 
\begin{align}
    [u]_i\,[v]_i &:= \sum_{i=1}^d [u]_i\,[v]_i = u^\top v \label{eq.einstein}
    ~,
\end{align}
where the first $:=$ in \eqref{eq.einstein} defines the summation notation and 
 the second $=$ relates this sum to the linear-algebra operation on two column-oriented vectors ($d\times 1$ matrices).
All indices appear either once (unsummed) or twice (summed) but never more than twice:
\begin{align}
\underbrace{[u]_i\,[v]_i\,[w]_i}_{\mathclap{\text{indices can only appear once or twice!}}} &= \text{undefined}
    ~.
    \end{align}
If there is one unsummed index in an expression, as in $[u]_i\,[v]_i\,[w]_j$, then the result will be an O($d$)-equivariant vector:
\begin{align}
    [u]_i\,[v]_i\,[w]_j &\equiv \left(\sum_{i=1}^d [u]_i\,[v]_i\right)\,w_j \equiv (u^\top v)\,w
    ~.
\end{align}
If there are $\ell$ unsummed indices, then the expression is an O($d$)-equivariant order-$\ell$ tensor.

This notation also reveals that the directionality of the output is entirely encoded in the input vectors and their combination.
This is counter-intuitive given some physics expressions, such as the electromagnetic force law explained in \eqref{eq.biot}; it contains a cross product that typically requires a particular coordinate system to determine the direction of the result.
In this case, the direction of the force cannot depend on the coordinate system; the second cross product saves us, but it is not immediately clear how. 
With Einstein notation, we can express the vector triple product in a coordinate-free way.
We first rewrite the $d=3$ cross product (pseudo-vector product) $a\times b$ in terms of the maximally anti-symmetric rank-3 tensor in $d=3$ (the Levi-Civita symbol) $\epsilon_{ijk}$:
\begin{align}
    f(u, v, w) &= (u\times v)\times w
    \\
    f_n &= [u]_j\,[v]_k\,[w]_m\,\epsilon_{ijk}\,\epsilon_{imn}
    ~.
\end{align}
The properties of the anti-symmetric tensor products are such that this product can be re-written as
\begin{align}
    f_n &= [u]_j\,[v]_k\,[w]_m\,[\delta_{jm}\,\delta_{kn} - \delta_{jn}\,\delta_{km}]
    \\
    f_n &= [u]_i\,[v]_j\,[w]_i - [u]_j\,[v]_i\,[w]_i
    \\
    f(u, v, w) &= (u^\top w)\,v - (v^\top w)\,u
    ~,
\end{align}
where $\delta_{ij}$ is the Kronecker delta.
We thus see that the cross product can be expressed entirely in terms of scalar products ($u^\top w$ and $v^\top w$) and the input vectors ($v$ and $u$), and that this can only be done by employing the anti-symmetric tensor.

Einstein notation is often credited with making physical-law expressions more compact or brief.
But what's important about the notation is that if the rules are obeyed, the notation can produce only equivariant objects in the theory.

Under these summation rules, all O($d$)-invariant scalar expressions that can be made from polynomial expressions of vectors will include only terms that use even numbers of vectors.
All these terms can be rearranged to be written as products of invariant scalar products.
That is, any scalar function that can be written as a polynomial of vectors (or a function of such polynomials) can be written in terms only of available scalar products.
This result corresponds very directly to Lemma \ref{lemma:1}.
For concreteness, here is an example 4-vector scalar form, written in summation notation, reordered as an inner product of tensors or as a simple product of scalars:
\begin{align}
\underbrace{[u]_i\,[v]_j\,[w]_i\,[z]_j}_{\text{4-vector scalar polynomial term}}&=\underbrace{([u]_i\,[v]_j)\,([w]_i\,[z]_j)}_{\text{inner product of order-2 tensors}}=\underbrace{([u]_i\,[w]_i)\,([v]_j\,[z]_j)}_{\text{product of scalars}}
~.
\end{align}

On the other hand, all O($d$)-equivariant vector expressions that can be made from polynomial expressions of vectors will include only terms that use odd numbers of vectors.
These terms can be rearranged by the vector with the unsummed index.
Once they are rearranged this way, the expression becomes the input vectors times polynomial expressions of scalar products.
This demonstrates that any O($d$)-equivariant vector expression constructible in the notation will lie in the subspace spanned by the input vectors.
It also demonstrates that the O($d$)-invariant scalar coefficients multiplying those vectors must themselves be constructible from polynomials (or functions of polynomials, it turns out) of scalar products.
What is also true but to the best of our knowledge doesn't appear explicitly stated in the physics literature is that every O($d$)-equivariant polynomial vector function can be expressed in Einstein notation. This result correspond very directly to Proposition \ref{prop.invariance}.

When the metric is non-trivial, we must face the covariant/contravariant distinction, where vector components might be pre-multiplied by the metric or not.
In this context, there are components $[u]_i$ of the covariant vectors and components $[u]^i$ of the contravariant vectors; the Einstein summation notation rules obtain the additional rule that repeated indices must belong to covariant-contravariant pairs:
\begin{align}
    [u]_i\,[v]^i &\equiv \sum_{i=1}^d [u]_i\,[v]^i \equiv \sum_{i=1}^d \sum_{j=1}^d [u]_i\,[\Lambda]^{ij}\,[v]_j
    \\
    [v]^i &\equiv \sum_{j=1}^d [\Lambda]^{ij}\,[v]_j
    ~,
\end{align}
where the $[\Lambda]^{ij}$ are the components of a $d\times d$ Hermitian metric tensor.
In the case of the Lorentz group, $d=4$ and the metric is diagonal, with elements $(-1,1,1,1)$ on the diagonal.
In the case of the curved spacetime of general relativity, $d=4$ and the metric is (in general) a function of spatial position and time.

We expect that these results will have generalizations for scalar, vector, and tensor functions of scalar, vector, and tensor inputs.

\section{Connections with low-rank matrix completion} \label{app.matrix.completion}
Given a rank $d$, $n\times n$ matrix $M$, Example 4 in \cite{pimentel2016characterization} shows that $M$ is almost always uniquely determined by the entries $\Omega(M):=M_{i,i+s}$ $i=1,\ldots, n$, $s=0,\ldots ,d$, considering the indices ``wrap around'' (i.e., $M_{i,n+s}$ corresponds to $M_{i,s}=M_{i,n+s(\mathrm{mod }\; n)}$). 
In particular invariant functions $f:(\mathbb R^d)^n\to \mathbb R^\ell$, can be expressed as \eqref{eq.invariant}
\begin{equation}
    f(v_1,\ldots, v_n) = g ( M) = \tilde g(\inner{v_{i}}{v_{i+s  (\mathrm{mod}\; n) }}_{ 1\leq i\leq n,\; 0\leq s \leq d}) =\tilde g( \Omega(M)),
\end{equation}
where $M$ is either $V^\top V$ or $V^\top\Lambda\, V$.
Therefore, any function $\tilde g: \Omega(M) \to \mathbb R^\ell$ uniquely determines $f$ on almost all possible inputs $v_1,\ldots, v_n$. This observation provides a parameterization for Lorentz and orthogonal-invariant and equivariant functions with a small number of scalars. 

We note that the results from \cite{pimentel2016characterization} don't assume the matrix $M$ is positive semi-definite, only that it is low rank. This is useful since $M$ is not positive semi-definite in the Minkowski case. 
We also remark that the learning of the functions  \eqref{od.equivariant} and \eqref{eq.invariant} can be done from $\Omega(M)$ without ever needing to compute all the scalars.

One disadvantage of the subset of scalars proposed in this section is that the sampling procedure $\Omega(M)$ is not permutation invariant. One of our future goals is to find a set of permutation invariant scalars that are universally expressive.

\section{Model-building considerations}\label{app.model}

The results in Section~\ref{sec.equivariance} provide a simple characterization of all scalar functions and vector functions that satisfy important symmetries (see Table~\ref{tab.groups}) for classical physics and special relativity.
Here we discuss how we might use this to design and build a model.

First we need to identify the inputs, what groups are acting, and how the group action affects the inputs and outputs.
For example, the translation group acts differently on position vectors than displacement vectors (see Table~\ref{tab.actions}) than other kind of vectors (velocities, accelerations, fields, etc).
Moreover, permutations don't typically act on the entire set of inputs but on tuples of inputs like the individual-particle charge, position, and velocity $(q_i, r_i, v_i)$ in the electromagnetic example in Section~\ref{sec.examples}.

The causal and group-action structure of the model must correspond to the structure of the physics problem.
For example, suppose we are interested in solving an $n$-body problem in which we predict the trajectories of a set of interacting charged particles.
The target functions $h_i()$, which predict the position of particle $i$ at time $t$, can be described as functions of that particle's charge, position, and velocity, and the set of charges, positions, and velocities of all the others.
The group actions (rotation, translation, and permutation) can be written as:
\begin{align}
h_i( q_i, Q\,r_i, Q\,v_i, (q_j,Q\,r_j,Q\,v_j)_{j=1; j\neq i}^n)
&= Q\,h_i( q_i, r_i, v_i, (q_j,r_j,v_j)_{j=1; j\neq i}^n)
\\
h_i(q_i, r_i + w, v_i, (q_j,r_j+w,v_j)_{j=1; j\neq i}^n)
&= h_i( q_i, r_i, v_i, (q_j,r_j,v_j)_{j=1; j\neq i}^n) + w
\\
h_i(q_i, r_i , v_i, (q_{\sigma_i(j)},r_{\sigma_i(j)},v_{\sigma_i(j)})_{j=1; \sigma_i(j)\neq i}^n)
&= h_i( q_i, r_i, v_i, (q_{j},r_{j},v_{j})_{j=1; j\neq i}^n)\label{eq:foo}
\\
h_{\sigma(i)}(q_{i}, r_{i}, v_{i}, (q_{j},r_{j},v_{j})_{j=1; j\neq i}^n)
&= h_{i}( q_{i}, r_{i}, v_{i}, (q_{j},r_{j},v_{j})_{j=1; j\neq {i}}^n)\label{eq:bar}
~,
\end{align}
where $\sigma$ is any permutation on $n$ elements, and $\sigma_i$ is a permutation that fixes $i$, i.e., $\sigma_i(i)=i$.
The primary results of this paper imply that, in this $n$-body case, the rotation symmetries can be enforced by constructing the invariant scalar products, and building functions thereof.
The translation symmetries can be enforced by taking positional differences with respect to $r_i$ prior to constructing the scalars.
The first permutation symmetry \eqref{eq:foo} can be enforced by working on a set-based neural network (or graph neural network), and the second permutation symmetry \eqref{eq:bar} can be enforced by taking all the functions $h_i=h$ to be identical.
In general, graph-based or set-based methods are probably good frameworks for physics problems \cite{cranmer2020lagrangian, sanchez-gonzalez2019hamiltonian}, and the dynamics can probably be implemented with a form of message-passing.

\end{document}